\RequirePackage{etoolbox}
\documentclass[12pt,final]{alt2021}  %
\newcommand{\qed}{\jmlrBlackBox}

\usepackage[capitalize]{cleveref}

\usepackage[utf8]{inputenc}
\usepackage[english]{babel}
\usepackage{latexsym}
\usepackage{mathrsfs}
\usepackage{mathtools}
\usepackage{bm}
\usepackage{datetime}
\usepackage{tikz}
\usepackage{listings}
\usepackage{mdframed}
\usepackage{pgfplots}
\usepackage{pgfplotstable}
\usepackage{dsfont}
\usepackage{color}
\usepackage{colortbl}
\usepackage{pifont}
\usepackage[bf,font=small,singlelinecheck=off]{caption}

\usepackage{microtype} %
\usepackage{float}
\usepackage{xfrac} %
\usepackage{xspace}

\renewcommand{\phi}{\varphi}

\usepackage{nicefrac}
\usepackage{wrapfig}

\usepackage[normalem]{ulem}

\renewcommand{\S}{\mathcal{S}}

\newcommand{\A}{\mathcal{A}}
\newcommand{\M}{\mathcal{M}}

\DeclareMathOperator{\Dists}{\mathcal{M}_1}

\newcommand{\bR}{\mathbb{R}}

\DeclareMathOperator{\poly}{poly}

\definecolor{emerald}{rgb}{0.31, 0.78, 0.47}

\usepackage{enumerate}

\usepackage{aliascnt}

\usepackage{xspace}

\usepackage{enumerate}

\newcommand{\absg}[1]{\left|#1\right|}
\newcommand{\abs}[1]{|#1|}

\newcommand{\E}{\mathbb E}
\newcommand{\EE}[1]{\mathbb E[#1]}
\newcommand{\EEg}[1]{\mathbb E\left[#1\right]}

\newcommand{\ip}[1]{\langle #1 \rangle}
\newcommand{\bip}[1]{\left\langle #1 \right\rangle}
\newcommand{\ipg}[1]{\bip{#1}}
\newcommand{\norm}[1]{\|#1\|}
\newcommand{\R}{\mathbb{R}}
\newcommand{\N}{\mathbb{N}}
\newcommand{\cA}{\mathcal{A}}
\newcommand{\cM}{\mathcal{M}}
\newcommand{\cC}{\mathcal{C}}

\newcommand{\cF}{\mathcal{F}}

\newcommand{\cH}{\mathcal{H}}

\newcommand{\cP}{\mathcal{P}}

\newcommand{\cR}{\mathcal{R}}
\newcommand{\cW}{\mathcal{W}}
\newcommand{\cS}{\mathcal{S}}
\newcommand{\cT}{\mathcal{T}}
\newcommand{\cX}{\mathcal{X}}

\newcommand{\one}[1]{\mathbb{I}\{#1\}}

\newcommand{\ass}
{\stepcounter{equation}
\tag{{\color{red} A\theequation}}}

\DeclareMathOperator{\supp}{supp}

\renewcommand{\epsilon}{\varepsilon}

\newcommand{\floor}[1]{\left\lfloor {#1} \right\rfloor}

\newcommand{\zeros}{\mathbf{0}}

\DeclareMathOperator*{\argmax}{arg\ max}

\newcommand{\dhpow}[1]{\left(\frac{1-\gamma}{2\gamma}\right)^{#1}}
\newcommand{\dhpows}[1]{\left(\tfrac{1-\gamma}{2\gamma}\right)^{#1}}

\newcommand{\SA}{\cS \times \cA}
\newcommand{\stin}{s_{\text{in}}}
\newcommand{\bbP}{\mathbb{P}}

\newif\ifsup\suptrue

\usepackage{algorithm}
\usepackage{algpseudocode}

\newtheorem{assumption}[theorem]{Assumption}

\usepackage{times}
\usepackage{newtxmath}

\title[Exponential Lower Bounds for Planning in MDPs]{Exponential Lower Bounds for Planning in MDPs With Linearly-Realizable Optimal Action-Value Functions}
\altauthor{%
 \Name{Gell\'ert Weisz} \\
 \addr DeepMind, London, UK\\
 \addr University College London, London, UK
 \AND
 \Name{Philip Amortila}\\
 \addr University of Illinois, Urbana-Champaign, USA%
 \AND 
 \Name{{Cs}aba {Sz}epesv\'ari} \\
 \addr DeepMind, Edmonton, Canada\\
 \addr University of Alberta, Edmonton, Canada
}

\begin{document}

    \maketitle

\begin{abstract}
We consider the problem of local planning in fixed-horizon and discounted Markov Decision Processes (MDPs) 
with linear function approximation and a generative model
under the assumption that the optimal action-value function lies in the span of a feature
map that is available to the planner. Previous work has left open the question of whether there exist sound planners that need only $\poly(H,d)$ queries
regardless of the MDP, where $H$ is the horizon and $d$ is the dimensionality of the features.
We answer this question in the negative: we show that any sound planner must query at least
$\min(e^{\Omega(d)},\Omega(2^H))$ samples in the fized-horizon setting and $e^{\Omega(d)}$ samples in the discounted setting.
We also show that for any $\delta>0$, the least-squares value iteration algorithm with
$\tilde{\mathcal{O}}(H^5 d^{H+1}/\delta^2)$ queries can compute a $\delta$-optimal policy in the fixed-horizon setting.
We discuss implications and remaining open questions.
\end{abstract}

\section{Introduction}

Much research in the theory of planning (and learning) in large-scale Markov Decision Processes (MDP)
with function approximation 
revolves around the conditions that are necessary for query-efficient planning.
In the setting we consider, a planner is given access to a feature map that maps state-action pairs to $d$-dimensional (feature) vectors, and
 interacts with a simulation model (also known as a generative model) of the MDP to find a ``good action'' at any given state.
The planner is promised that the unknown optimal action-values at any state-action pair can be written as the inner product between the features at that state and an unknown parameter vector. 
As is well known, actions that are optimal to take in a state are those that maximize
the optimal action-values at the state.
Thus, if the planner can compute a good approximation to the unknown parameter vector, it could return a good action at the initial state (and perhaps even at all the states).
The hope then is that regardless the size of the state and action spaces, a planner may be able to find a good action while only interacting with the simulator $\mathcal{O}(\poly(H,d))$ times, where $H$ is the horizon of the MDP.

Much has been written about planning (and learning) in the presence of (linear) function approximation with additional assumptions. Yet, the basic question:
\begin{center}
\emph{Is realizability of the optimal action-value function enough to guarantee query-efficient learning?}
\end{center}
had remained unanswered so far.
In this paper, we answer this question in the negative by proving that 
for any $\eta\in \left(0,\frac12-\frac2{\log_2(d-1)}\right]$,
the worst-case
query-complexity is at least
\begin{align*}
\Omega
\left( 
	\min\left\{ 
		e^{(d-1)^{2\eta}/8} , 
		2^{-H} d^{H\left(\frac{1}{2}-\eta\right)}
		\right\}
\right)\,.
\end{align*}
In particular, this implies a lower bound of $\min(e^{\Omega(d)},\Omega(2^H))$. Thus, in general, the query complexity of a planner which returns a ``good'' solution may be exponential in $d$ or exponential in $H$ (for a precise statement of this result, see Theorem~\ref{thm:lb}). Since our results hold under the generative setting, by extension, they also apply to the more general online setting (where a simulator is not given). 
Our proof also translates into the discounted MDP setting, where we prove a query complexity of $e^{\Omega(d)}$. For ease of presentation, this paper primarily examines the fixed-horizon setting, presenting the relevant changes to prove the lower bound for the discounted MDP setting in Section~\ref{sec:discounting}.

Inspired by the recent work of \citet{Du_Kakade_Wang_Yan_2019}, our lower bound construction uses the Johnson-Lindenstrauss lemma \citep{johnson1984extensions} to create a large set of nearly-orthogonal feature vectors. At a high-level, each state of our construction has this large set of features available as actions (see Figure~\ref{fig:mdp-illustration}). The parameter of the optimal action-value function is then chosen to be the feature of one of these actions -- by realizability this entails that one action per state will have higher value while the others are nearly identical. At the final stage of the MDP, this gap between the optimal action and its non-optimal counterparts will be exponentially small. The gap is then increased multiplicatively as the planner approaches the initial stage. There are additional subtleties needed to ensure realizability, which distinguishes our approach from the bandit-like lower bound construction of \citet{Du_Kakade_Wang_Yan_2019}. 

We note that our lower bound construction is deterministic, save for random rewards obtained at the last time-step of each episode. This is in contrast to the result of \citet[Theorem 1]{Wen_Roy_2013}, which establishes that for fully deterministic MDPs, 
there exists a planner that chooses only at most $d$ times suboptimal actions out of any number of episodes.
Together with our construction,
this suggests a stark separation between planning/learning
 in deterministic and in stochastic environments. 

\subsection{Related work} 
The idea of using function approximation to help solving large-scale
MDPs originates in the 1960s \citep[e.g.,][]{BeKaKo63}:
These early works provided experimental evidence that in MDPs with large (or even infinite) state spaces, the optimal value function can be well approximated with the linear combination of a few basis functions, 
which in turn
encouraged work to explore how such basis functions could be used
to design efficient planning algorithms whose compute cost is
\emph{independent of the size of the state space}
and depends mildly on the number of basis functions and the planning horizon.
The seminal paper of \citet{SchSei85} gave general, ``least-squares'' versions of the basic dynamic 
programming methods (value iteration, policy iteration and linear programming) that relied on
the basis functions. 
However, no analysis was provided.

In an independent line of work, 
\citet{kearns2002sparse} noticed that  
if a planner that is given simulator access to the MDP is asked to return a good action 
only at some fixed state (provided as part of the input to the planner),
the query (and computational) complexity of planning can already 
be made independent of the size of the state space.
However, without further extra structure (such as the presence of helpful basis functions), 
the query complexity of local planning turns out to be exponential in the planning horizon \citep{kearns2002sparse}. 

Merging the two directions of research gives rise to the central question 
of efficient planning in MDPs, namely 
whether the aforementioned exponential dependence can be
avoided by assuming the presence of ``helpful'' basis functions. 
When the action-value function of \emph{all} policies
are well-represented by some linear combination of the basis functions,
``fitted policy iteration''-type algorithms have recently been shown to achieve
polynomial query (and computational) complexity 
\citep{YW19,LaSzeGe19}. 
These results extend to the case when the basis functions induce a worst-case
 approximation error $\epsilon>0$ over the policies, as long as the 
 planner is required to return $\mathcal{O}(\sqrt{d}\epsilon)$-optimal policies only.
Interestingly, demanding $\mathcal{O}(\epsilon)$-optimal policies worsens the query complexity
to be \emph{exponential}, namely to $\Omega(e^{\Omega(d)}\wedge 2^H)$
\citep{Du_Kakade_Wang_Yan_2019}. %

While these results give a (nearly) complete characterization of the query and computational complexity
under the said assumption, they left open the question whether a similar result may hold true 
under the milder assumption that only the optimal value function is well-approximated by the basis functions.
Positive results for this question have been obtained under a number of additional assumptions. In the online setting, \citet{Wen_Roy_2013} provides a low-regret guarantee for deterministic MDPs  (when both the rewards and transitions are deterministic). %
This was later extended to ``low-variance'' MDPs by \cite{Du_Luo_Wang_Zhang_2019} under an additional gap assumption, which requires knowledge of the minimum separation in value between any optimal action and the second best action. Their sample complexity also scales in the inverse gap. %
Further positive results have been obtained for MDPs with a linear reward function and a low-rank transition matrix \citep{Jin_Yang_Wang_Jordan_2019} (a condition which entails the above assumption of \citet{LaSzeGe19}),
and for MDPs with low ``Bellman ranks'' \citep{Jiang_Krishnamurthy_Agarwal_Langford_Schapire_2017}. 
In the planning setting, \citet{Du_Kakade_Wang_Yan_2019} give a query complexity result for least-squares value iteration which scales as $\mathcal{O}(\poly(H,d))$ provided that the inverse gap is treated as a fixed parameter and is itself $\mathcal{O}(\poly(H,d))$, which is a restrictive condition. Additionally, \citet{Sharriff_Szepesvari_2020} obtains polynomial bounds under $q^\star$-realizability with the additional assumption that the features for all state-action pairs are inside the convex hull of the features at $\mathcal{O}(\poly(H,d))$ state-action pairs. 

Lastly, we highlight the concurrent work of \cite{wang2020statistical}, which establishes a similar exponential lower bound for the analogous question in the setting of \textit{offline RL}. \footnote{In offline RL, the agent is not allowed to interact with the environment but is instead given a dataset of transition tuples drawn from a certain distribution.} In their setting, they assume linear-realizability and an additional assumption related to the data-generation process. While their construction does not work in the generative setting, there are a number of interesting similarities, notably a scaling phenomenon which geometrically reduces values throughout the horizons.

The rest of the paper is organized as follows. 
In the next section, we provide preliminary definitions and introduce notation. %
Section~\ref{sec:locplanning} introduces the formal problem definition, which we call local planning under $q^\star$-realizability with linear function approximation. The exponential lower bound, together with its proof, is presented in Section~\ref{sec:lower}. 
Section~\ref{sec:upper} shows that a least-squares value iteration with $\mathcal{O}(H^5d^{H+1}/\delta^2)$ queries is able to guarantee $\delta$-optimal policies. The paper is concluded with a discussion of the results and the remaining open problems Section~\ref{sec:disc}.

\section{Preliminaries}\label{sec:prelim}
Let $\R$ denote the set of real numbers, and for a positive integer $i$, let $[i] = \{1,\dots,i\}$ be the set of integers from $1$ to $i$. We let $\N_+ = \{1, 2, \dots \}$ denote the set of positive integers.
We write $\Dists(\cX)$ for the set of probability measures on a measurable space $(\cX,\cF)$ where $\cF\subset 2^{\cX}$ is a $\sigma$-algebra over $\cX$. Random quantities are denoted by capital letters, but some capital letters denote non-random quantities.
We let $a\wedge b = \min(a,b)$.
For a probability measure $\mu$ over a topological space, we let $\supp(\mu)$ denote its support.
We use $\norm{f}_\infty = \sup_{x\in \cX} |f(x)|$ 
to denote the supremum norm of a function $f:\cX \to \R$.
We use $\one{A}$ to denote the indicator of a set $A$.

We consider MDPs given by a tuple $M=(\cS,\cA,Q)$, where $\cS$ and $\cA$ are the set of states and actions, respectively, $Q=(P_a: \cS \rightarrow \Dists(\R\times\cS))_{a\in \cA}$ are Markov transition kernels from $\cS$ to $\R\times \cS$ \citep{Put94}.
We assume that $\cS$ is a measurable subset of some Euclidean space,%
\footnote{
We only need a restriction on what these sets are so that we can refer to the set of all MDPs concerned.
Clearly, we could allow much more generality here.
This matters for the upper bound: for the lower bound finite sets are sufficient.
}
while we assume that $\cA$ is finite. 
An MDP describes an environment that an agent interacts with in a sequential manner in discrete time steps. %
At time $t$, the agent observes the state $S_t\in \cS$, takes an action $A_t$ to transition to state $S_{t+1}$ while incurring the reward $R_{t+1}$ where $(S_{t+1},R_{t+1}) \sim Q_{A_t}(\cdot|S_t)$. 

In the fixed, finite-horizon setting, the agent-environment interaction happens in episodes, with each episode lasting for $H>0$ steps. For simplicity, we consider the variant where every episode starts with a fixed initial state $s_1\in \cS$. %
The goal of the agent is to maximize the total expected reward incurred in the episodes. 
Let $r_a(s) = \int Q_a(dr,\cS|s) r$ and $P_a(ds'|s) = Q_a(\R,ds'|s)$.
For $h\in[H]$, define $\cS_h$ as the set of states accessible from $s_1$ in $h-1$ steps.
Thus, $\cS_1 = \{ s_1\}$ and $\cS_{h+1} = \{ s'\in \cS \,:\, \exists a\in \cA, s\in \cS_h \text{ s.t. } s'\in \text{supp} P_a(\cdot|s) \}$. 

A memoryless %
policy in this setting takes the form $\pi = (\pi_h)_{h\in [H]}$ where $\pi_h: \cS_h \rightarrow \Dists(\A)$ is a probability kernel from the states in $\cS_h$ to actions. 
Given $h\in [H]$ and a state $s\in \cS_h$, the value 
 $v^\pi_h(s)$ of $\pi$ for stage $h$ and state state $s$ 
is defined to
be the total expected reward incurred when $\pi$ is deployed beginning at stage $h$
 from state $s$.
Denoting by $\E_\pi$ the expectation operator induced over state-action sequences $(S_1,A_1,\dots,S_H,A_H)$ by the interconnection of $\pi$ and $M$, formally we have
\begin{align*}
v^\pi_h(s) = \E_\pi\left[\sum_{t=h}^H r_{A_t}(S_t) \mid S_h = s \right]\,,
\end{align*}
 We call $v^\pi = (v^\pi_h)_{h\in [H]}$ the \textbf{value function} of policy $\pi$. The value of the \textit{optimal} policy $\pi^\star$ is found to satisfy:
\begin{align}
q^\star_h(s,a) &= r_a(s) + \int v^\star_{h+1}(s') P_a(ds'|s)  \,,\label{eq:bellman1}\\
v^\star_h(s) &= \max_{a\in \cA} q^\star_h(s,a)\,\label{eq:bellman2},
\end{align} 
with 
\begin{align*}
v^\star_{H+1}(s) = 0\,, \qquad s\in \cS_{H+1}\,.
\end{align*}
Here, $v^\star=(v^\star_h)_{h \in [H]}$ is called the optimal value function and $q^\star = (q^\star_h)_{h \in [H]}$ is called the optimal action-value function
(we exclude $v_{H+1}^\star$, as this is identically zero). 
Without loss of generality, we will assume that $\cS = \cup_{h=1}^{H+1} \cS_h$.
An optimal $H$-horizon behavior for an agent in the MDP is one where, for each stage $h\in [H]$ of an episode, 
for each state $s\in \cS_h$, the agent takes an action that maximizes $q_h^\star(s,\cdot)$.
Further, the solution to the above equations is unique \citep{Put94}.

We will also need the notion of the \textbf{action gap} of an action $a$: At stage $h \in [H]$ and 
state $s\in \cS_h$, the action gap of action $a$ is defined to be $\Delta^\star_h(s,a) = v^\star_h(s) - q^\star_h(s,a)$.
When needed, we use $\Delta^\star_{h,M}$ 
to indicate the dependence of the action gap on the MDP $M$.
Finally, for $f:\SA\to\R$ and a memoryless policy $\pi$, we will use $f(s,\pi)$ as the shorthand for $\sum_{a\in \cA} \pi(a|s) f(s,a)$.
Similarly, we will also use $P_{\pi}$ to denote the Markov transition kernel defined using 
$P_{\pi}(ds'|s) = \sum_{a\in\cA} \pi(a|s) P_a(ds'|s)$.

\section{Planning with features realizing $q^\star$}

In this section we introduce the formal problem definition. We consider the \emph{local planning problem} in the presence of an MDP simulator (or ``generative model''
 \citealt{kearns2002sparse}) and a feature map $\varphi=\left(\varphi\right)_{h\in[H]}$, $\varphi_h:\cS\times \cA \to \R^d$, that, for some stage $h$, maps state-action pairs to $d$-dimensional vectors with some finite $d>0$.
In the local planning problem, the planner is given a stage $h \in [H]$ and an input state $\stin \in \cS_h$, and must return a (possibly random) action $A \in \cA$. Roughly, we ask that $A$ has low suboptimality as measured by the action gap (a more precise statement is postponed to Definition \ref{def:sound}).
The planner has access to the set of states, the decomposition $\cS = \cup_{h=1}^{H+1}\cS_h$,
 the set of actions, and the feature map. Furthermore, they are promised that $q^\star$ lies in the span of the features. This promise is captured by the following assumption:
\begin{assumption}[Realizability]\label{ass:realizability}
There exists a vector $\theta^\star\in \R^d$ such that for any $h\in [H]$ and state-action pair $(s,a)\in \cS_h \times \cA$, %
\begin{align}
q^\star_h(s,a) = \ip{ \varphi_h(s,a), \theta^\star }\,.
\label{eq:realizability}
\end{align}
If this holds, we say that the features $\phi$ \textit{realize} the optimal action-value function $q^\star$.
\end{assumption}
\begin{definition}\label{def:realizable-mdp}
A pair $(M,\phi)$ is called $q^\star$-realizable if $M$ is an MDP whose optimal action-value function $q^\star$ is realized by the feature map $\phi$.
\end{definition}

We further assume that the MDP has rewards bounded by $[0,1]$ almost surely: 
\begin{assumption}[Bounded rewards]\label{ass:bounded-mdp}
For any time $t$, the reward $R_{t+1}$ received after taking any action in any state, $\bbP(R_{t+1}\not\in [0,1])=0$.
\end{assumption}

\begin{definition}\label{def:bounded-realizable-mdp}
We write $\mathcal{M}_{H,d}$ for the set  of MDP-feature map pairs $(M,\phi)$ where $M$ has bounded rewards, the horizon is $H$, $\phi$ has dimension $d$, and $(M,\phi)$ is $q^\star$-realizable.
\end{definition}

\paragraph{Interaction between the planner and the simulator, implemented policy}
\label{sec:locplanning}
\begin{figure}[tb]
\begin{center}
\includegraphics[width=0.75\textwidth]{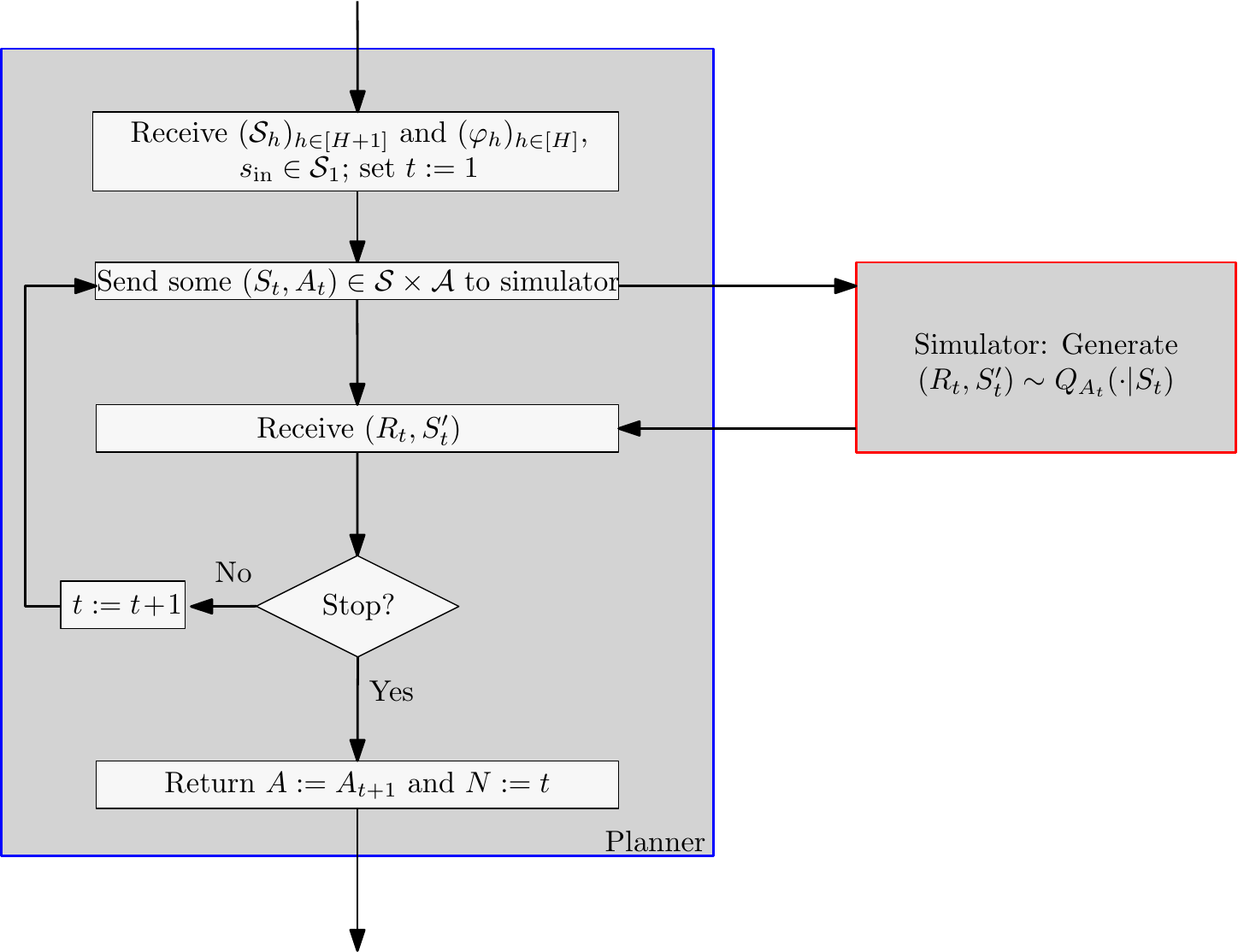}
\end{center}
\captionsetup{justification=centering}
\caption{Data flow between the planner and the simulator}
\label{fig:int}
\end{figure}
A planner is given a stage $h\in [H]$ and an input a state $\stin\in \cS_h$. The planner can then access the MDP $M= (\cS,\cA,Q)$ by sending \emph{queries} to the underlying simulator. A query is defined to be a state-action pair $(s,a)\in \SA$. Given the query $(s,a)$, the simulator returns a pair $(R,S)\sim Q_a(\cdot,\cdot|s)$. 
Based on the result, the planner can then send in a new state-action pair, and so on and so forth.
Thus, in general, the query will be random, as it depends on (random) data received previously by the planner from the simulator.  For the sake of generality, we also allow the planner to inject extra randomness into the planning process.
The planner eventually must decide to stop (say after $N<\infty$ queries) 
and output an action $A\in \cA$.
Formally, a planner together with $\phi$, $\cS\times \cA$, $\stin$ and $h$ determines two sequences of probability kernels,
$\rho = (\rho_t)_{t\ge 1}$ and $\lambda = (\lambda_t)_{t\ge 1}$, 
where $\rho_t$ is a probability kernel from $\cH_{t-1}:=(\cS \times \cA \times \R \times \cS \times \{0,1\})^{t-1}$ to $\cS \times \cA$,
while $\lambda_t$ is a probability kernel from $\cH_{t-1}\times (\cS \times \cA \times \R \times \cS)$ to $\{0,1\}$.
For $t\ge 1$, $\lambda_t$
determines whether the planner should continue with the queries. Without loss of generality, we use $1$ to denote the choice that the planner continues.
Similarly, for $t\ge 1$, $\rho_t$ determines the query to be used in step $t$ if the planner has not stopped in time step $t-1$, and if it has just stopped, the action component sampled from $\rho_t$ determines the action to be returned by the planner. 
The interaction between the planner and the simulator is shown in \cref{fig:int}.
In general, both the returned action $A$ and the number of queries $N$ are random. 
A planner is well-formed if $N<\infty$ holds with probability one for all MDPs that the planner is designed to interact with.

When the planner is well-formed, $A$ is well-defined.
We write $\pi_h(\cdot|\stin)$ for the distribution of $A$. These distributions altogether define a memoryless policy $\pi = (\pi_h)_{h\in [H]}$: this is the policy that will be followed if the planner is called sequentially on the trajectory where the actions taken are those chosen by the planner.
Thus, $\pi$ is the policy that is effectively \emph{implemented} by the planner.
Note that the planner does not return $\pi$, nor does it return $\pi_h(\cdot|\stin)$.
In particular, when the action set is large, $\pi_h(\cdot|\stin)$ may be too large to even write down, let alone $\pi$. Note also that $\pi_h(\cdot|\stin)$ itself is not a random variable as this distribution already accounts for the randomness of the simulator and that of the planner.

\paragraph{The probability distribution induced over interaction sequences}
\label{sec:inducedpd}
In what follows we will need a probability distribution $\bbP$ over the infinitely long interaction sequences $\cH = (\cS \times \cA \times [0,1] \times \cS \times \{0,1\})^{\N_+}$ 
that the interconnection between a planner and a simulator induces. 
Denote by $(S_1,A_1,R_1,S_1',C_1, S_2,A_2,R_2,S_2',C_2, \dots )\in \cH$ such a sequence (formally, these are the coordinate functions from $\cH$ to the respective spaces)
and for $t\ge 1$ 
introduce the abbreviation $H_{t-1} = (S_1,A_1,R_1,S_1', C_1, \dots,  S_{t-1},A_{t-1},R_{t-1},S_{t-1}', C_{t-1})$.
The induced probability distribution $\bbP$ satisfies the following properties: 
For any $t\ge 1$, $\bbP$-almost surely,
\begin{itemize}
\item Next query follows $\rho_t$: $\bbP(\, (S_t,A_t)\in \,\cdot\, | H_{t-1} ) = \rho_t(\,\cdot\,|H_{t-1})$;
\item Data received follows $Q$:   $\bbP(\, (R_t,S_t')\in \,\cdot\, | H_{t-1},S_t,A_t ) = Q_{A_t}(\,\cdot\,|S_t)$;
\item Continuation follows $\lambda_t$: $\bbP( C_t = 1 | H_{t-1},S_t,A_t, R_t,S_t') = \lambda_t(1|
H_{t-1},S_t,A_t, R_t,S_t')$.
\end{itemize}
Note that with this notation, $N = \min_{t\ge 1} \one{C_t=0}$. Also, to save on notation, 
without loss of generality,
we let the output of the planner 
be $A_{N+1}$, so the distribution of the output is also determined by $\rho$.
Since the planner takes as input $\phi$, $\stin$ and $h$, $\bbP$ also depends on these quantities, in addition to depending on the MDP $M= (\cS,\cA,Q)$.
We use $\bbP_{M,\phi,\stin,h}^{\cP}$ 
to signify this dependence when needed.
We will also $\E_{M,\phi,\stin,h}^{\cP}$ to denote the corresponding expectation operator.
The Ionescu-Tulcea theorem %
guarantees that $\bbP$ with these properties exist (see 
\citealt{Ion49}, or Theorem 6.17 in  the book of \citet{Kal06}).
Finally, note that we defined $\bbP$ for sequences of infinite length regardless of the planner by assuming that $\rho$ is defined for arbitrary histories, including those that have $C_t=0$ possibly multiple times. Of course, there is no loss of generality here: statements concerning soundness and query complexity of planners are only concerned with histories up to the first time when $C_t=0$. However, for our proofs, it will be convenient for probabilities to be assigned to events formed of sequences of infinite length. Note that with this notation, we have that for any $a\in \cA$,
$\pi_h(a|\stin)=\bbP_{M,\phi,\stin,h}^{\cP}(A_{N+1}=a)$ .

\paragraph{Sound planners and query cost}
The goal of the planner is to arrive at a policy $\pi$ that is nearly optimal regardless of stages and states.
We define the \emph{suboptimality gap} of a policy $\pi$ by $$\delta^\pi:=\sup_{h\in [H],s\in \cS_h} \{v_h^\star(s)-v_h^\pi(s)\}.$$
When the dependence on the MDP is needed, we use $\delta^\pi_M$ instead of $\delta^\pi$. The following definition introduces two notions of soundness for the policies implemented by planners.
 
 \begin{definition}[$\delta$-soundness and $(\delta,\zeta)$-transition-soundness]\label{def:sound}
 Recall that $\mathcal{M}_{H,d}$ is the set of $q^\star$-realizable and bounded $(M,\phi)$ pairs with horizon $H$ and dimension $d$ (Definition~\ref{def:bounded-realizable-mdp}).%

 Let $\cP$ be a planner and $\pi_{\cP,M,\phi}$ be the policy which is implemented by the planner when given $\varphi$ and upon interaction with the simulator of $M$. Let $\delta \geq 0$ and $\zeta \in [0,1]$.
 \begin{itemize}
     \item We say $\cP$\textit{ is $\delta$-sound} if for \emph{any} $(M,\phi) \in \cM_{H,d}$, the suboptimality of $\pi_{\cP,M,\phi}$ satisfies 
     \[ 
     \delta^{\pi_{\cP,M,\phi}} \leq \delta\,;
     \]
     \item We say that $\cP$\textit{ is $(\delta,\zeta)$-transition-sound} if for any $(M,\phi) \in \cM_{H,d}$, and for all $h \in [H], s \in \S_h$
     \begin{align*}
        \sum_{a\in \cA} 
        \one{\Delta^\star_{h,M}(s,a)\ge \delta}
        \pi_h( a \,|\,s) \le \zeta\,.
\end{align*}
 \end{itemize}
 \end{definition}

For the lower bound we will find it more convenient to consider $(\delta,\zeta)$-transition-soundness. The following proposition shows that these two notions of soundness are roughly equivalent:
\begin{proposition}\label{prop:delta-convert-tilde-delta} 
For $\zeta\in(0,1]$ and $(M,\phi) \in \M_{H,d}$, a planner that is $(\delta,\zeta)$-transition-sound is also $H\delta+H(H+1)\zeta/2$-sound, and a planner that is $\delta$-sound is also $(\delta/\zeta,\zeta)$-transition-sound. %
\end{proposition}

In particular, this means that
if $\cP_\delta$ is the set of all $\delta$-sound planners (for some model class) and $\cP_{(\delta,\zeta)}$ is the set of all $(\delta,\zeta)$-transition-sound planners then
$\cP_\delta \subset \cP_{(\delta/\zeta,\zeta)}$ and 
$\cP_{(\delta,\zeta)} \subset \cP_{H\delta + H(H+1) \zeta/2}$.
Thus, if for some $\cP'$, $\cP_{(\delta/\zeta,\zeta)} \cap \cP' = \emptyset$ then we also have that
$\cP_{\delta}\cap \cP'=\emptyset$.

\begin{proof}
Note that the soundness of a planner only depends on how well the policy it implements is doing in the MDP.
Hence, it suffices to show that for any MDP in $(M,\phi)\in\mathcal{M}_{H,d}$ and policy $\pi$, $\delta$ and $\zeta$ the following hold:
\begin{enumerate}
\item If for any stage $h\in [H]$ and state $s\in \cS_h$, 
$\sum_a 
\one{\Delta^\star_h(s,a)\ge \delta}
\pi_h(a \,|\,s)\le \zeta$ then $\delta^\pi  \le  H \delta  + H(H+1)\zeta/2$; %
\item 
If $\delta^\pi \le \delta$ then for any stage $h\in [H]$ and state $s\in \cS_h$, 
$\sum_a \one{\Delta^\star_h(s,a)\ge \delta/\zeta} \pi( a\,|\, s)\le \zeta$.
\end{enumerate}
For the first claim fix $h\in [H]$, $s\in \cS_h$. Then,
\begin{align*}
v^\star_h(s) - v^\pi_h(s) 
&= v^\star_h(s)-\E_\pi \left[\sum_{t=h}^H r_{A_t}(S_t)\mid S_h = s \right] \\
&= v^\star_h(s)-\E_\pi \left[\sum_{t=h}^H q^\star_t(S_t,A_t)-v^\star_{t+1}(S_{t+1}) \mid S_h = s \right]\\
&= \E_\pi \left[\sum_{t=h}^H v^\star_t(S_t)-q^\star_t(S_t,A_t) \mid S_h = s \right]\\
& = \sum_{t=h}^H \E_\pi \left[ \Delta^\star_t(S_t,A_t) \one{ \Delta^\star_t(S_t,A_t)< \delta }  \mid S_h = s \right]\\
&\quad+ \sum_{t=h}^H \E_\pi \left[ v^\star_t(S_t)-q^\star_t(S_t,A_t) \one{ \Delta^\star_t(S_t,A_t)\ge \delta} \mid S_h = s \right]
\\ 
&\le \delta H + \zeta \frac{H(H+1)}{2}\,,
\end{align*}
where the last step used that $\pi$ is $(\delta,\zeta)$-transition-sound and that
$v^\star_t(S_t)-q^\star_t(S_t,A_t)\le v_t^\star(S_t) \le H-t+1$ due to the bounded rewards (since $(M,\phi)\in\mathcal{M}_{H,d}$, Assumption~\ref{ass:bounded-mdp} holds). %
The result follows by taking the supremum over $h$ and $s$ of both sides and noting that we get $\delta^\pi$ on the left-hand side.

To prove the second claim let $\pi$ be such that $\delta^\pi\le \delta$ and fix $h\in [H]$ and $s\in\cS_h$. Let $A\sim \pi_h(\cdot|s)$.
Then,
 $\delta^\pi\ge v_h^\star(s)-v_h^\pi(s)\ge \sum_{a\in\cA} \pi_h(a|s) \left(v_h^\star(s)-q^\star(s,a)\right)=\E[v_h^\star(s)-q^\star(s,A)]$.
By Markov's inequality, $\bbP(v_h^\star(s)-q^\star(s,A) \ge \E[v_h^\star(s)-q^\star(s,A)]/\zeta ) \le \zeta$, 
which is what we wanted to show.
\end{proof}

Besides its soundness, the second figure of merit of a planner is the expected number of queries that was used for planning.
\begin{definition}[Query cost of a planner]\label{def:query-cost}
The \textit{query cost} of the planner, $\cC(\cP,H,d)$, is the worst-case 
expected number of queries over any $q^\star$-realizable $(M,\phi)$ with horizon $H$ and dimension $d$, and any input state: 
\begin{align*}
\cC(\cP,H,d)=\sup_{(M,\varphi) \in \mathcal{M}_{H,d}} \sup_{\stin,h} \E^{\cP}_{M,\varphi, \stin,h}[ N ]\,,
\end{align*}
where $\E^{\cP}_{M,\varphi, \stin,h}[ \cdot ]$ is the expectation under the distribution of planner interactions as defined above.
\end{definition}

In the following section, we give a lower bound on the query cost of any planner on the set of $q^\star$-realizable $(M,\phi)$ pairs. 

\begin{remark}
Note that in an episode of length $H$, the planner is called $H$ times.
The extreme of this is when a planner is used in an infinite horizon or discounted MDP, 
as in this case the planner would be called infinitely many times and could 
thus submit infinitely many queries to the simulator, even when the number of queries in each call is limited.
Note also that the definition implicitly forbids planners that have a global memory where they collect information
about the MDP they interact with (across calls with different $(\stin,h)$ pairs).
However, this is only to simplify the presentation, as our lower bound proof is based on the planner making a mistake at the first call, where such global memory would be empty.

\end{remark}

\section{Lower bound}
\label{sec:lower}
This section is devoted to proving the main result of the paper:
\begin{theorem}
\label{thm:lb}
For any $d$ and $H$ large enough and any 
\begin{align}
0<\eta \le \frac{1}{2}-\frac{2}{\log_2(d-1)}\,,
\label{eq:etarangecond}
\end{align}
any planner $\cP$ that is $\frac{9}{128}$-sound
on the set of $q^\star$-realizable $H$-horizon local planning problems with linear  
function approximation with features in $\R^d$,
the planner's worst-case query complexity $\cC(\cP,H,d)$ satisfies %
\begin{align*}
\cC(\cP,H,d)
&=\Omega\left(\min\left\{ e^{(d-1)^{2\eta}/8} , 2^{-H} d^{H\left(\frac{1}{2}-\eta\right)} \right\}\right)\,.
\end{align*}
\end{theorem}
Thus, the lower bound says that the query complexity is either exponential in the dimension, or it is exponential in the planning horizon (when the dimension is large). While we gave the result using the $\Omega(\cdot)$ formalism to minimize clutter, in our proof we compute the lower bound in an explicit form. In fact, the lower bound is shown to hold for $d\ge 18$ and $H\ge 1$,
although no attempt is made to optimize the lower bound on $d$. 
The above result implies the following:
\begin{corollary}\label{cor:eta-choice}
By choosing $\eta=\frac{1}{2}-\frac{2}{\log_2(d-1)}$, 
\begin{align*}\cC(\cP,H,d)=\min(e^{\Omega(d)},\Omega(2^H))\,.
\end{align*}
\end{corollary}

\subsection{Proof: Main ideas}
First note that
thanks to Proposition~\ref{prop:delta-convert-tilde-delta}, it suffices to show the query lower bound
for $(\delta/\zeta,\zeta)$-transition-sound planners with some $\zeta>0$.

At a high level, if $\alpha_H$ is the planner's ``measuring accuracy'' at the final stage $H$, we construct MDPs where all except some optimal action $a$ have $q^\star$-value in $[0,\alpha_H]$ at the last stage.  
All these values look like $0$ to the planner. 
Inspired by the work of \citet{Du_Kakade_Wang_Yan_2019}, the idea of our construction will be to use the Johnson-Lindenstrauss lemma to construct features.
This allows us to give action $a$
the optimal value 
 $d^{\frac12-\eta}\alpha_H$. Unless $a$ is played, the planner's measuring accuracy for stage $H-1$ is thus $\alpha_{H-1}=d^{\frac12-\eta}\alpha_H$. At stage $H-1$, the same argument is repeated, until finally we derive a suboptimality gap at the first stage that is exponentially larger than the measuring accuracy $\alpha_H$. An important deviation from the construction of 
 \citet{Du_Kakade_Wang_Yan_2019} is that we need to ensure realizability with the growing gaps.
 To be able to do this, we let the optimal action leave the normal states (which like in their paper are arranged in a tree) and reach a stream of special states with no rewards and a straight path to the end.
Then, to hide the identity of the optimal action, we need to have many actions.
Note that the dynamics are deterministic and the rewards are also deterministic except for the last stage and the amount of randomness here is chosen carefully so that planners that focus on the last stage need many interactions before identifying the optimal action at the first stage.

The rigorous proof is based on constructing $k\approx e^{d^{2\eta}/8}$ MDPs, $(M_a)_{a\in \cA}$
with a shared action set $\cA = [k]$ and a shared state space $\cS$ and initial state $s_1\in \cS$
and feature map $\phi$ %
such that 
in MDP $M_a$, $a$ is the optimal action in all states and in particular in state $s_1$,
 $(M_a,\phi)$ are $q^\star$-realizable for any $a\in \cA$ and the rewards are in $[0,1]$.
The planner will face one of these MDPs, the identity of which is hidden.
Given $n \approx k/4 \wedge (d/2)^{H(1/2-\eta)}$, which one should think of as an intended upper bound 
on the number of interactions between the planner and the MDP, 
the MDPs and the feature map are chosen so that the identity of $a$ is difficult to establish in $n$ planning steps while the action gap at the initial state is large. 
In particular, for any $a\in \cA$, 
\begin{align}
\min_{a'\ne a}\Delta^\star_{1,M_a}(s_1,a')\ge \frac{1}{4}\,.
\label{eq:subgaps}
\end{align}
This means that for any $\delta\le 1/4$ and $\zeta\in (0,1)$, 
a $(\delta,\zeta)$-transition-sound 
planner run on $M_a$ needs to figure out the identity of $a$ with at least $1-\zeta$ probability.
To hide the identity of $a$, we use random rewards in the last stage with low signal to noise ratio.

To argue that planners using at most $n$ interactions have a hard time 
to identify the optimal action,
a ``test MDP'' $M_0$ is constructed, which, by construction, is ``symmetric'' over the actions.
This MDP is used to find out any bias a planner with ``essentially no information'' may have: If a planner is under-utilizing some action $a$ in this MDP, 
we will show that the planner will fail on $M_a$ in identifying $a$ as the action to be taken at $s_1$.
To be able to show this, we make sure that $M_0$ shares the structure of the other MDPs apart from the fact that in it all actions behave the same ($M_0$ is invariant to permutations of the action set).
As a result, letting $A_{1:n}$ denote the random sequence of actions taken by the planner $\cP$, it will hold that
the probability assigned to $a\not\in A_{1:n}$ when $\cP$ is interacting in $M_0$
lower bounds the same probability under $M_a$:
\begin{align}
\bbP_{M_a}^{\cP}( a\not\in A_{1:n} ) \ge  \frac{3}{4} \,  \bbP_{M_0}^{\cP}( a \not\in A_{1:n} )
\qquad \text{for all planners } \cP \text{ and } a\in [k]\,.
\label{eq:problb}
\end{align}
Here, $\bbP_{M}^{\cP}$ is the probability distribution jointly induced by $\cP$ and $M$ (with the input $\phi$, $\stin = s_1$ and $h=1$), as defined in Section \ref{sec:locplanning}. To minimize clutter the dependence on $\phi,\stin$ and $h$ is not shown.
The following lemma shows that this construction allows one to give a lower bound
on the query complexity of sound planners:
\begin{lemma}
\label{lem:mainlemma}
Let $d,H\ge 1$, $k\ge 2$ and let $n\le \lfloor k/4 \rfloor$.
Assume that there exists $(M_a,\phi)\in\mathcal{M}_{H,d}$ for any $a\in\cA$, and $M_0$ such that 
Eq.~(\ref{eq:subgaps}) and Eq.~(\ref{eq:problb}) hold.
Take any $(\frac{1}{4},\frac{9}{32})$-transition-sound planner $\cP$ 
for the $q^\star$-realizable fixed horizon planning problem.
Then there exists $a\in \cA$ such that 
$\E_{M_a}^{\cP}[N]\ge  \frac{9}{32}\, n$.
\end{lemma}
To use this result, we will need to show that 
Eq.~(\ref{eq:subgaps}) and Eq.~(\ref{eq:problb}) can be satisfied for ``large'' values of $n$ as suggested above.
\begin{proof}
Let $k$ and $n$ be as in the statement of the lemma.
Fix the planner $\cP$ whose input is $\stin =s_1$, $\phi$ and $h=1$.
Since $A_{1:n} = (A_1,\dots,A_n)$ has at most $n\le k/4$ actions,
by the pigeonhole principle, there exists $a\in \cA$ such that
$\bbP_{M_0}^{\cP}( a \in A_{1:n} )\le n/k\le 1/4$.
Take such an action $a$.
We have $\bbP_{M_0}^{\cP}( a \not\in A_{1:n} )\ge 3/4$.
This, together with Eq.~(\ref{eq:problb}) gives that
\begin{align}
\bbP_{M_a}^{\cP}( a\not\in A_{1:n} ) >  \left(\frac{3}{4}\right)^2\,.
\label{eq:pmalb}
\end{align}
Let $A = A_{N+1}$ be the output of the planner.
Since the planner is $(\frac{1}{4},\zeta:=\frac{9}{32})$-transition-sound and $(M_a,\phi)\in\mathcal{M}_{H,d}$,
\begin{align*}
\zeta 
& \ge \bbP_{M_a}^{\cP}( \Delta^\star_{1,M_a}(s_1,A) \ge \frac{1}{4} )
= \bbP_{M_a}^{\cP}( a\ne A  ) \\
& \ge \bbP_{M_a}^{\cP}( a\ne A_1,\dots,a\ne A_{N+1}, N \le n-1  )
\ge \bbP_{M_a}^{\cP}( a\ne A_1,\dots,a\ne A_{n}, N \le n-1  ) \\
&= \bbP_{M_a}^{\cP}( a\not\in A_{1:n}, N \le n-1  )\,,
\end{align*}
where the first equality used Eq.~(\ref{eq:subgaps}).
From $P(A \cap B) = P(A) - P(B^c \cap A)\ge P(A)-P(B^c)$, we have
\begin{align*}
\zeta 
&\ge
\bbP_{M_a}^{\cP}( a\not\in A_{1:n}, N \le n-1  )
\ge
\bbP_{M_a}^{\cP}( a\not\in A_{1:n}  )
-
\bbP_{M_a}^{\cP}( N \ge n  )
>
\left(\frac{3}{4}\right)^2
-
\E_{M_a}^{\cP}[N]/n\,,
\end{align*}
where the last inequality follows from
Eq.~(\ref{eq:pmalb})
and that from Markov's inequality,
$\bbP_{M_a}^{\cP}( N \ge n  )\le \E_{M_a}^{\cP}[N]/n$.
Reordering and plugging in the value of $\zeta$ gives
\begin{align*}
\E_{M_a}^{\cP}[N]/n \ge \frac{1}{2}\left(\frac{3}{4}\right)^2\,.
\end{align*}
The desired claim follows by algebra.
\end{proof}

\paragraph{Organization of the rest of the section}
In the next section, we define the states, actions and the feature map.
In Section~\ref{sec:trans} we define the transitions, the parameter vector and the rewards for the MDPs $(M_a)_a$. In this section we establish that these MDPs are well-defined and the rewards indeed lie in $[0,1]$. This is followed by Section~\ref{sec:real}, where we show the realizability of the optimal action-value function by the features and the parameter vectors given earlier.
In Section~\ref{sec:showing-hardness} we show that Eq.~(\ref{eq:subgaps}) is satisfied, along with some additional properties that we assumed along the way for the various constants involved in the construction.
In Section~\ref{sec:m0def} we define $M_0$ and in Section~\ref{sec:similarity} we show that 
Eq.~(\ref{eq:problb}) is satisfied by $M_0$ and $(M_a)_a$.
Finally, the formal proof of Theorem~\ref{thm:lb} based on Lemma~\ref{lem:mainlemma} and the construction of the previous sections is given in Section~\ref{sec:mainproof}.

\subsection{States, actions and feature map}

Let $d\ge 18$ and fix some $k>0$ to be chosen later.
We let $\cA=[k]$. 
The set of states is $\cS$ is the disjoint union of two sets: The ``game-over states'' $\cF=\{f_2,\dots,f_{H+1}\}$
and the ``tree states'', $\cT$.
The tree states, as the name suggests, \emph{can be thought of} as nodes in a $k$-ary complete rooted tree of depth $H$, with the edges in the tree labeled by the actions. 
Any node in this tree can be identified with the sequence of actions that, in the tree (but not necessarily in the MDPs to be constructed later), leads to the node.
However, we only allow action sequences with non-repeated actions. Thus,
\begin{align*}
\cT = \cup_{h=0}^{H-1} \Bigl\{ (a_1,\dots,a_h)\in \cA^h \,:\, \text{ for any } 1\le i\ne j\le h\,, a_i \ne a_j \Bigr\}\,.
\end{align*}
Formally, we then have
\begin{align*}
\cS = \{ f_2,\dots,f_{H+1} \} \cup \cT = \cF \cup \cT\,,
\end{align*}
where the union is disjoint.
The level decomposition of the state space follows the tree structure.
Denoting by $\bot$ the empty sequence (the unique element of $\cA^0$), %
$\cS_1 = \{ \bot \}$, $\cS_2 = \{ f_2 \} \cup \cA$, $\dots$, $\cS_{H}=\{ f_{H} \} \cup \cA^{H-1}$, $\cS_{H+1}=\{ f_{H+1} \}$.

For convenience, we abbreviate $(a_1,\dots,a_i) \in \cA^i$ as $a_{1:i}$. For $s=(a_1,\dots,a_i)\in \cA^i$ and $a\in \cA$ we let $sa\in \cA^{i+1}$ be $(a_1,\dots,a_i,a)$.
By abusing notation we also let $a_{1:0}=\bot$ denote the unique element of $\cA^0$.
For convenience, the sequence notation is further abused by identifying the sequence with the underlying subset of actions. This allows us to write $a\in a_{1:i}$, which means $a \in \{ a_1,\dots,a_i\}$. We also allow $a\in f_h$, which is defined to be false.
For $s\in \cS$, we let $|s|$ denote the level of state $s$, the unique index such that $s\in \cS_{|s|}$.
Specifically,
for $s\in \cT$ if $s\in \cA^i$ then $|s|=i+1$ and for $s\in \cF$ with $f = f_h$ then $|s|=h$.

To construct the feature map $\varphi$ corresponding to our MDP, we employ the next lemma, which is a consequence of the Johnson--Lindenstrauss lemma \citep{johnson1984extensions,dasgupta2003elementary,Du_Kakade_Wang_Yan_2019}: 
\begin{lemma}
\label{lem:jl}
For any $\gamma>0$ and positive integer $d'$ such that 
$d'\ge \lceil 8 \log(k)/\gamma^2 \rceil$,
the $\ell^2$-unit sphere in $\R^{d'}$ has 
  $k$ distinct vectors $(v_a)_{a\in [k]}$
 such that 
for all $a,b \in [k]$, $a\neq b$, it holds that $|\ip{v_a,v_b}| \leq \gamma$.
\end{lemma}
\begin{proof}
See the proof of \cite[Lemma A.1]{Du_Kakade_Wang_Yan_2019}.
\end{proof}

To use the lemma we choose the tuning parameter $\eta>0$ and set 
\begin{align}
\gamma=(d-1)^{-\frac{1}{2}+\eta}\,.
\label{eq:gammdef}
\end{align}
Then, the conditions of the lemma are satisfied with $d'=d-1$ and
\begin{align}
k=\floor{e^{\tfrac{(d-1)^{2\eta}}{8}}}\,.
\label{eq:kdef}
\end{align}
While $\gamma$ in general should be thought of as a small value, 
it will be useful to put a specific upper bound on it.\footnote{$\gamma$ is not to be confused with the discounting factor. When dealing with discounted MDPs (\cref{sec:discounting}), we denote the discount factor by $\alpha$.}
As it turns out, the following constraint will be particularly useful:
\begin{align}
\gamma\le1/4\,.
\label{eq:gamma_qb}
\end{align}
To satisfy this constraint, we restrict the range of $\eta$: 
\begin{align}
0<\eta \le \frac{1}{2}-\frac{2}{\log_2(d-1)}\,.
\label{eq:etarange}
\ass
\end{align}
Note that owing to $d\ge 17$, 
the range of $\eta$ is nonempty
and one can indeed verify that for $\eta$ in this range, $\gamma\le1/4$ indeed holds. Here, and in what follows, we use the convention of putting the letter `A' in front of an equation number to mark those relations that remain to be shown.

Let $(v_a)_i$ be the set of $d-1$-dimensional vectors on the $\ell^2$-sphere whose existence is guaranteed by Lemma~\ref{lem:jl}. 
Let $(\sigma_{s,a})_{s\in \cT,a\in \cA,a\not\in s}$  be positive constants to be chosen later.
Let $c_1,\dots,c_{H}$ be constants defined as follows:
\begin{align}
c_h &= \frac{1}{2}+\frac{1+\gamma}{2}\sum_{l=1}^{H-h} \dhpow{l}\,.
\end{align}
Note that empty sums are defined to be zero, hence $c_H = 1/2$. 
With this, the feature map $\left(\phi\right)_{h\in[H]}$, $\phi_h:\SA \to \R^d$ is defined as follows: for $h\in[H]$, $s\in \cS$, and $a\in \cA$, 
\begin{align}
\phi_h(s,a) = 
\begin{cases}
\zeros\,, & \text{if } s\in\cF \text{ or } \left(s \in \cT \text{ and } a\in s \right)\,;\\
\left(c_{h}, \,\dhpows{H-h+1}\sigma_{s,a} v_a^\top \right)^\top\,, & \text{otherwise}\,.
\end{cases}\label{eq:phi-choice}
\end{align}
Note that $\phi$ is well-defined: the second branch applies only if $s\in \cT$.
We will see that $\sigma_{s,a}$ satisfies 
\begin{align}\label{ineq:sigma-bounds}
\gamma\le\sigma_{s,a}\le 1. 
\ass
\end{align}
The ideas behind the constants $c_{h}$, scaling with $\sigma_{s,a}$, and why we require it to be within this range is explained once the MDP is fully defined, in Remark~\ref{rem:intuition-on-c-and-sigma}.
The high level goal of this construction is to make sure that realizability holds,
$a^\star$ is always optimal, and the bias term $c_{h}$ will ensure that at the first stage we have a large action gap, while the identity of $a^\star$ remains hidden. 

\subsection{Transitions, parameter vector and rewards}
\label{sec:trans}
We now construct a family of MDPs $(M_{a,\epsilon})_{a\in \cA,\epsilon>0}$ with state space $\cS$ and action set $\cA$.
Here, $\epsilon$ is a parameter whose value we will choose later (the role of $\epsilon$ is to allow some rewards in the MDP to be randomized with ``signal-to-noise ratio'' $\mathcal{O}(\epsilon)$).

For $a^\star\in [k]$ and $\epsilon>0$ fixed,
the transition and rewards in the MDP $M_{a^\star,\epsilon}$ are as follows:
The state transitions are deterministic.
Once a game-over state $f_h$ is reached with $2\le h \le H$,
the agent can only transition to the next game-over state $f_{h+1}$ regardless of the action taken.
State $f_{H+1}$ is an absorbing state: Any action taken here leads to $f_{H+1}$.
Consider now a tree state $s\in \cT$ at level $1\le h\le H-1$.
Taking action $a\in \cA$ in $a$ leads to  $f_{h+1}$
if either $a=a^\star$ or $a\in s$.
Taking any other action leads to the next tree node, $sa\in \cA^{h+1}$.
When in a leaf node, that is when $s\in \cA^{H-1}$, taking any action leads to $f_{H+1}$.
Formally, letting $g:\SA \to \S$ denote the function that gives the next state for any given state-action pair, we have
\begin{align}
g(s,a) = 
	\begin{cases}
	f_{H+1}\,, & \text{if } |s| \ge H\,;\\
	f_{|s|+1}\,, & \text{if } |s|<H \text{ and } \left(s\in \cF\setminus \{f_H\} \text{ or } a=a^\star \text{ or } a\in s \right)\,;\\
	s a\,, & \text{otherwise}\,.
	\end{cases}\label{eq:state-transition-def}
\end{align}

The reward structure is dictated by the choice of the parameter vector:
\begin{align}
\theta^\star = \epsilon\left(1, \, v_{a^\star}^\top \right)^\top\,.
\label{eq:thchoice}
\end{align} %
Let the reward distribution given a state $s$ and action $a$ be $\cR_a(\cdot|s)$.
Then, using $\delta_{x}$ to denote the Dirac distribution with point mass at $x$,
\begin{align}
\cR_a(\cdot|s) = 
	\begin{cases}
	\delta_{\ip{\phi_{|s|}(s,a),\theta^\star}}\,, & \text{if } s\in \cT \text{ and } a = a^\star\,;\\
	\text{Ber}(\mu_a(s))\,, & \text{if } s\in \cA^{H-1} \text{ and } a \ne a^\star\,;\\
	\delta_0\,, & \text{otherwise}\,.
	\end{cases}\label{eq:reward-def}
\end{align}
In words, the reward is deterministically zero 
when either the state is a game-over state or the state is a non-leaf tree state
but the action is not $a^\star$.
When the action is $a^\star$, the reward is deterministic and is equal to $\ip{\phi_{|s|}(s,a),\theta^\star}$.
Finally, for any other action $a\ne a^\star$ in a leaf tree state $s\in \cA^{H-1}$,
the reward is drawn from a Bernoulli distribution with mean $\mu_a(s)$, which is chosen as follows:
\begin{align}
\label{eq:mu-def}
\mu_a(s) & = \epsilon\sigma_{s,a} \, \dhpows{} \ipg{ v_a ,v_{a^\star}} + \epsilon/2\,.
\end{align}
Note that the value of $\mu_a(s)$ is selected such that for $s\in \cA^{H-1}$ and $a \ne a^\star$, $\mu_a(s)=\ip{\phi_{|s|}(s,a),\theta^\star}$.

To make the rewards well-defined, we need $\mu_a(s)\in [0,1]$.
For this, let us show that
\begin{align}\label{ineq:bernoulli-small}
0\le \mu_a(s) \le \epsilon\,.
\end{align}
Let $x =  \epsilon\sigma_{s,a} \dhpows{}\ipg{ v_a ,\, v_{a^\star}}$.
Using that $|\sigma_{s,a}|\le 1$ (by Eq.~(\ref{ineq:sigma-bounds})), and the upper bound in $\ip{v_a,v_{a^\star}}$ for $a\ne a^\star$ from  Lemma~\ref{lem:jl}, we get
$|x|\le \epsilon \gamma \left(\frac{1}{2\gamma}-\frac{1}{2}\right)\le \epsilon/2$. This proves Eq.~(\ref{ineq:bernoulli-small})
under the bound on $\sigma_{s,a}$.
Hence, the rewards are well-defined as long as
\begin{align}
\epsilon\le 1
\label{eq:eps-less-than-one}
\ass
\end{align}
holds.

Figure~\ref{fig:mdp-illustration} illustrates this MDP. 
To summarize, the \emph{only} times that rewards are given are either when $a^\star$ is played, or when any action is played in a non-game-over state at the \emph{final} stage.

Note that by construction, there can only be one non-zero reward received in any episode, whether this reward is given as a result of playing $a^\star$ (case 1), or because we are in the final stage (case 2) (to satisfy realizability according to (Assumption~\ref{ass:realizability}), the expectation of it will be the inner product between the features and weight vector ($\ip{\varphi_h(s,a),\theta^\star}$) to be constructed momentarily.
In case 1, this reward is deterministically a constant, whereas in case 2, it will be the sum of a constant and a Bernoulli with very low expectation.
At a high level, the feature map does not depend on $a^\star$, and unless $a^\star$ is played at some stage, none of the deterministic transitions depend on $a^\star$, so the agent can only learn about $a^\star$ through the low-expectation Bernoulli rewards.

\begin{figure}
\centering
\includegraphics[width=0.5\textwidth]{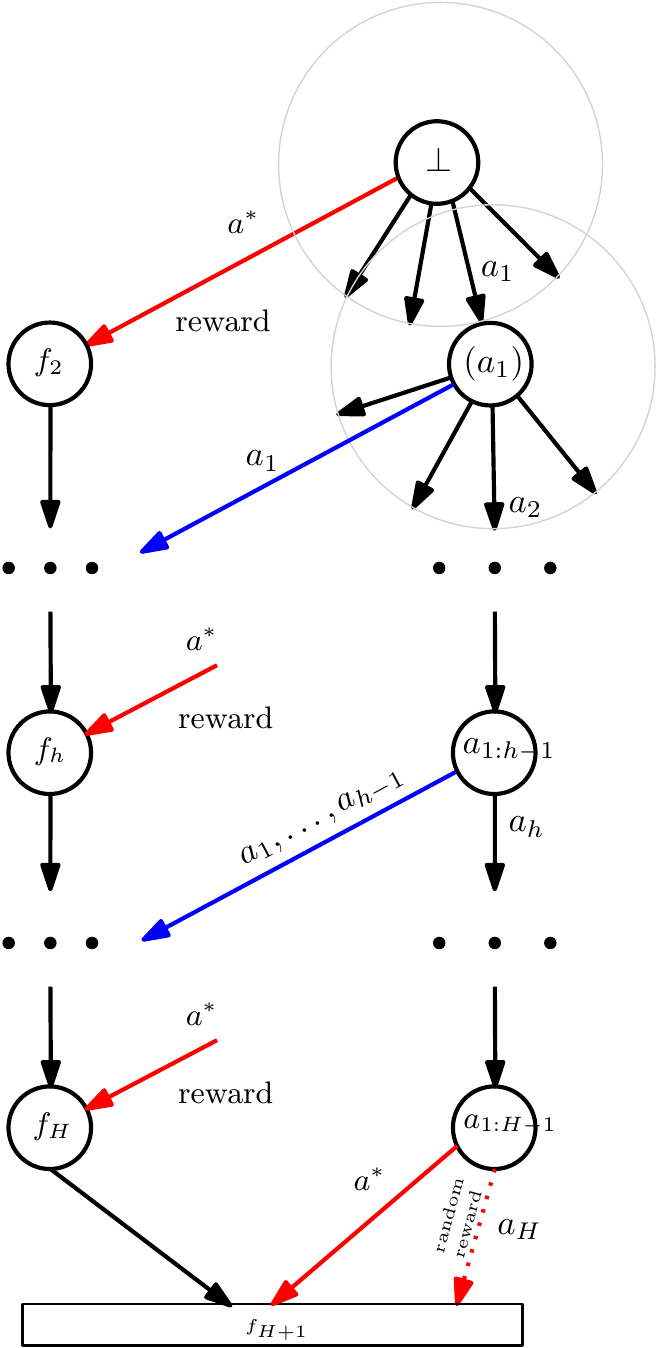}
\caption{Illustration of the MDP $M_{a^\star}$. 
The states on the right belong to the ``tree'': These are uniquely determined by sequences of actions.
The dynamics is deterministic. The initial state $s_1 = \bot$ and this is the state that is given as the input $\stin$ to the planner. Action $a^\star$, or actions that are repeated cause the next state to be a ``game-over state''. In these states (labeled by $f_2,\dots,f_{H+1}$, no rewards are obtained and the next state in $f_i$ is $f_{i+1}$ regardless of the action (except for $f_{H+1}$, where the next state is $f_{H+1}$). The difference between taking a repeated and the action $a^\star$ is that $a^\star$ gives a positive (large) reward, while repeated actions do not give rewards. 
The rewards in the bottom ensure that realizability holds but they are random and make the detection of $a^\star$ hard.
For more details, see the main text.
}\label{fig:mdp-illustration}
\end{figure}

\subsubsection{Defining $\sigma_{s,a}$ and satisfying Eq.~(\ref{ineq:sigma-bounds})}

We now define $(\sigma_{s,a})_{s \in \cT, a\not\in s}$.
For $a\in\cA$, let $\sigma_{\bot,a}=1$, and for $h>1$, $sa\in \cA^{h-1}\cap \cT$ with $a,a'\in \cA$ and $a'\not\in sa$, 
let
\begin{align}\label{eq:sigma-def}
\sigma_{sa,a'}=\sigma_{s,a}\tfrac{1-\gamma}{2\gamma}\ip{v_a,v_{a'}}+\tfrac{1+\gamma}{2} \,. 
\end{align}

To prove that Eq.~(\ref{ineq:sigma-bounds}) holds for all $s\in \cT$ and $a\in \cA$ with $a\not\in s$, 
we use induction on $h=|s|$. 
For $h=1$, $\sigma_{s,a}=1$ by definition. 
Let $h>1$ now and assume that Eq.~(\ref{ineq:sigma-bounds}) holds for any $s\in \cA^{h-1}\cap \cT$.
Take any $s'\in \cA^h \cap \cT$ and write it as $s' = sa$ with $s\in \cA^{h-1}$ and $a\in \cA$.
By assumption, $a\not\in s$.
Hence, by the induction hypothesis, $0\le \sigma_{s,a}\le 1$.
We have $\sigma_{sa,a'} = \sigma_{s,a} \ip{v_{a},v_{a'}} \frac{1-\gamma}{2\gamma} +\frac{1+\gamma}{2}$.
Since $a\ne a'$, $-\gamma \le \ip{v_a,v_{a'}}\le \gamma$.
Thus, $-\gamma \le  \sigma_{s,a} \ip{v_{a},v_{a'}}  \le \gamma$. 
Plugging this into the expression of $\sigma_{sa,a'}$ gives $\gamma \le \sigma_{sa,a'} \le 1$, proving the induction hypothesis for $s'$ and $a'$ and hence the desired claim follows by induction.

\begin{remark}\label{rem:intuition-on-c-and-sigma}
The reason  for choosing the value of $\sigma_{sa,a'}$ is to satisfy the Bellman equations as follows: if a suboptimal action $a\ne a^\star$ is played at state $s$ of stage $h$, $\phi_{h+1}(sa,a')$ 
for any next-state action $a'$ is scaled such that we pretend $a'$ is the optimal action, and make $\ip{\phi_{h+1}(sa,a'),\theta^{a'}}=\ip{\phi_h(s,a),\theta^{a'}}$ 
(where $\theta^{a'}$ is the value that $\theta^\star$ would take if $a^\star$ was $a'$). %
Crucially, the scaling does not depend on the optimal action, so reveals no information about it to the learner: any action can be the optimal one, and each $\sigma_{sa,a'}$ is calculated as if $a'$ was the optimal one.

We have to ensure $\sigma_{sa,a'}$ is in some range (Eq.~(\ref{ineq:sigma-bounds})), 
or else such scaling could change what the optimal action is at some stage,
leading to $\ip{\phi_{h+1}(sa,a'),\theta^\star}>\ip{\phi_h(s,a),\theta^\star}$.
Keeping $\sigma_{s,a}$ in the required range is indeed (indirectly) the role of $c_h$. In particular, if $c_h=0$, the above scaling breaks the Bellman equations. 
Furthermore, to see why $c_h=0$ fails for any scaling $\sigma_{s,a}$ that satisfies the Bellman equation, let us suppose our planner chooses an action $a$ from the starting state such that $\ip{v_a,v_{a^\star}}<0$ (in our construction this holds for a significant fraction of actions).
This leads to features $\phi_{h+1}(sa,a')$ for any $a'$ that also have a negative inner product with $\theta^\star$ (due to the Bellman equation). 
This reveals much information about $v_{a^\star}$, and therefore also about the identity of $a^\star$, 
while our proof strongly relies on extremely little information being revealed about $a^\star$.
\end{remark}

\subsubsection{Bounding the support of random rewards and satisfying Eq.~(\ref{eq:eps-less-than-one})}
\label{sec:rewbound}
We show that Assumption~\ref{ass:bounded-mdp}, that is that the support of rewards lies in $[0,1]$, 
is satisfied for $M_{a^\star,\epsilon}$. 

In the definition of rewards, there are three cases: The reward is either identically zero or a Bernoulli (in which cases there is nothing to be proven), or it is a deterministic value. In the latter case, the optimal action is taken ($a=a^\star$) in a tree state ($s\in \cT$). 
In this case the rewards take on the value $r_{a^\star}(s) = \ip{ \phi_{|s|}(s,a^\star),\theta^\star}$ deterministically.
Plugging in the definitions,
\begin{align}
\ip{ \phi_{|s|}(s,a^\star),\theta^\star} 
= 
\epsilon c_{|s|} +\sigma_{s,a^\star} \epsilon \dhpow{H-|s|+1}\,.
\label{eq:optrew}
\end{align}

Notice that $\frac{1-\gamma}{2\gamma}\ge\frac32$ (cf. Eq.~(\ref{eq:gamma_qb})).
Therefore, both terms are positive (hence $r_{a^\star}(s)\ge 0$), and decreasing as $|s|$, the level of the state, increases.
It remains to show the upper bound on these rewards. By monotonicity,
the largest value in Eq.~(\ref{eq:optrew}) will be less than one provided
\begin{align}
\epsilon\left( \frac12 + \frac{1+\gamma}{2}\sum_{l=1}^{H-1}\dhpow{l} + \dhpow{H}\right) &\le 1 \,.
\label{eq:egbound}
\end{align}
holds. As $\frac12 + \frac{1+\gamma}{2}\sum_{l=1}^{H-1}\dhpow{l} + \dhpow{H} < \sum_{l=0}^H \dhpow{l}$, this holds if 
\begin{align}
\epsilon \, \sum_{l=0}^H  \dhpow{l} \le 1\,.
\label{eq:epsilon-not-too-big}
\ass
\end{align}
Finally note that 
Eq.~(\ref{eq:epsilon-not-too-big}) implies that $\epsilon\le1$ (i.e., Eq.~(\ref{eq:eps-less-than-one})).

\subsection{Showing realizability}
\label{sec:real}
The goal of this section is 
to show that $(M_{a^\star,\epsilon}, \varphi)$ is $q^\star$-realizable  according to Definition~\ref{def:realizable-mdp} %
and the parameter $\theta^\star$ in the realizability definition can be chosen as shown in Eq.~(\ref{eq:thchoice}).
Together with Section~\ref{sec:rewbound}, this implies that $(M_{a^\star,\epsilon}, \varphi)\in\mathcal{M}_{H,d}$.

Let the optimal action-value function of $M_{a^\star,\epsilon}$ be $q^\star$. 
We show realizability (cf. Eq.~(\ref{eq:realizability})) by showing $q_h(s,a):=\ip{\phi_h(s,a),\theta^\star}$ satisfies the Bellman equations (Eqs.~(\ref{eq:bellman1}) and (\ref{eq:bellman2})). By uniqueness (Section~\ref{sec:prelim}), this implies $q_h=q^\star_h$.

Take $s\in \cS$, $a\in\cA$, and let $h=|s|$. We prove the statement by a case-analysis.
One of the following cases hold:
\begin{description}
\item[Case 1] $s$ is a game-over state ($s\in\cF$) or $a\in s$;
\item[Case 2] $s$ is not game-over ($s\in\cT$) and $a\not\in s$.
\end{description}
For Case 1, the Bellman equations trivially hold as $q_h(s,a)=0$ and no reward is given (neither immediately, nor for the rest of the episode). 
We subdivide Case 2 into three sub-cases:
\begin{description}
\item[Case 2.1] $a=a^\star$;
\item[Case 2.2] $h=H$ and $a\ne a^\star$;
\item[Case 2.3] $h<H$ and $a\ne a^\star$.
\end{description}
For Cases 2.1 and 2.2, there is an immediate reward given, the expectation of which matches $q_h(s,a)$ by definition. 
In both cases the next state, $g(s,a)$, is a game-over state.
Then, thanks to Case 1, $v_{h+1}(g(s,a))=0$, so the Bellman equation is satisfied in these cases.

What is left to show is that the Bellman equation is satisfied for Case 2.3.
In Case 2.3, $s\in \cT$, $a\not\in s$,  $h< H$, and $a\ne a^\star$. %
We show that the Bellman equation holds in two parts:
first, we will show that $a^\star$ is the optimal next-state action:
\begin{align}\label{eq:bellman-star-stays-optimal}
a^\star\in\argmax_{a'\in\cA} q_{h+1}(sa, a') \,.
\end{align}
Second, $q_h$ at $(s,a)$ satisfies the Bellman equation:
\begin{align}\label{eq:bellman-satisfied}
q_h(s, a) = \underbrace{r_a(s)}_{=0}+q_{h+1}(sa,a^\star) \,,
\end{align}
where we used Eq.~(\ref{eq:bellman-star-stays-optimal}).

To show Eq.~(\ref{eq:bellman-satisfied}), it suffices to show that
\begin{align*}
\ip{\theta^\star, \varphi_{h}(s,a)} 
- \ip{\theta^\star, \varphi_{h+1}(sa,a^\star)}=0\,.
\end{align*}
Plugging in definitions (noting that under Case 2.3 $\varphi_h(s,a)\ne \bm{0}$ and $\varphi_{h+1}(sa,a^\star)\ne \bm{0}$), we get that
\begin{align*}
\MoveEqLeft
\ip{\theta^\star, \varphi_{h}(s,a)} 
- \ip{\theta^\star, \varphi_{h+1}(sa,a^\star)}\\
&=
\left ( \epsilon c_{h} + \epsilon\dhpows{H-h+1} \sigma_{s,a}\ip{v_{a^\star}, v_a} \right)
- 
\left( \epsilon c_{h+1} + \epsilon\dhpows{H-h} \sigma_{sa,a^\star} \ip{v_{a^\star},  v_{a^\star}} \right)  \\
& =
\tfrac{1+\gamma}{2} \epsilon \dhpows{H-h} +
\epsilon\dhpows{H-h+1} \sigma_{s,a} \ip{v_{a^\star},  v_a}
-  %
\epsilon\dhpows{H-h}\sigma_{sa,a^\star}\\
& = 
\epsilon \dhpows{H-h}
\left\{
\tfrac{1+\gamma}{2} +
\epsilon\dhpows{} \sigma_{s,a} \ip{v_{a^\star},  v_a}
-
\sigma_{sa,a^\star} \right\}\,.
\end{align*}
Plugging in the definition of $\sigma_{sa,a^\star}$ gives the desired result.

To prove Eq.~(\ref{eq:bellman-star-stays-optimal}), 
pick any $a'\in \cA$ such that $a'\ne a^\star$. Our aim is to show that $q_{h+1}(sa, a^\star)\ge q_{h+1}(sa,a')$.
If $a'\in sa$, then $q_{h+1}(sa,a')=0$, and we are done 
because we have already shown that all the rewards are nonnegative (in fact,
we have even shown that Assumption~\ref{ass:bounded-mdp} holds) and
thus $q_{h+1}(sa,a^\star)\ge0$.
In the remaining case, and since $sa$ is not a game-over state under Case 2.3, $\phi_{h+1}(sa,a')\ne \bm{0}$ and $\phi_{h+1}(sa,a^\star)\ne \bm{0}$, and we can substitute their values as:
\begin{align*}
\MoveEqLeft q_{h+1}(sa,a^\star)-q_{h+1}(sa,a') \\
& =
\left( \epsilon c_{h+1} + \epsilon\dhpows{H-h} \sigma_{sa,a^\star} \ip{v_{a^\star}, v_{a^\star}} \right)
-
\left( \epsilon c_{h+1} + \epsilon\dhpows{H-h}\sigma_{sa,a'}\ip{v_{a^\star}, v_{a'}} \right) \\
&= 
\epsilon\dhpows{H-h} \left(\sigma_{sa,a^\star}
-
\sigma_{sa,a'}\ip{v_{a^\star}, v_{a'}}\right)\,.
\end{align*}
The right-hand side here is nonnegative as $|\ip{v_{a^\star},v_{a'}}|\le \gamma$ (due to $a^\star\ne a'$ and Lemma~\ref{lem:jl}), $\sigma_{sa,a'}\le 1$, and $\sigma_{sa,a^\star}\ge \gamma$ (by Eq.~(\ref{ineq:sigma-bounds})).
This finishes the proof that the Bellman equations hold in Case 2.3 and thus they hold for all state-action pairs. 
As such, $q_h=q^\star_h$, which shows that $q^\star_h$ is indeed realizable.

\subsection{Action gap lower bounds (Eq.~(\ref{eq:subgaps}))
and satisfying Eq.~(\ref{eq:epsilon-not-too-big})
}
\label{sec:showing-hardness}
Let us now turn to showing Eq.~(\ref{eq:subgaps}) with $s_1=\bot$,
that is, that for $a\ne a^\star$, the action gaps
\begin{align*}
\Delta^\star_1(\bot,a)=\ip{\phi_1(\bot,a^\star)-\phi_1(\bot,a),\theta^\star}
\end{align*}
are bounded from below by $1/4$.

Plugging in the definitions of $\phi_1$ and $\theta^\star$, we get
\begin{align*}
\ip{\phi_1(\bot,a^\star)-\phi_1(\bot,a),\theta^\star}
& = \epsilon \dhpow{H} (1-\ip{v_{a^\star},v_a})
\ge \epsilon \dhpow{H} (1-\gamma)\\
& \ge \frac34 \epsilon \dhpow{H} \,,
\end{align*}
where the first inequality follows from the choice of $(v_a)_a$
and the second follows because $\gamma\le1/4$ (cf. Eq.~(\ref{eq:gamma_qb})).
To get $\Delta^\star_1(\bot,a)\ge\frac{1}{4}$, we set 
\begin{align}
\epsilon=\frac{1}{3}\dhpow{-H} \,.
\label{eq:epdef}
\end{align}

Let us now show that with this choice of $\epsilon$, Eq.~(\ref{eq:epsilon-not-too-big}) is also satisfied. 
For this let $x=\frac{1-\gamma}{2\gamma}$. As before, $x\ge1.5$ due to $\gamma\le\frac{1}{4}$. 
Hence,
\begin{align*}
\epsilon\,\sum_{l=0}^H  \dhpow{l} 
= \frac{1}{3} x^{-H} \sum_{l=0}^H x^l=\frac{1}{3}\sum_{l=0}^H x^{-l}
<\frac{1}{3}\sum_{l=0}^\infty x^{-l} 
=\frac{1}{3}\cdot\frac{1}{1-x^{-1}}
\le\frac{1}{3}\cdot\frac{1}{1-\frac{2}{3}}=1.
\end{align*}

\subsection{Constructing $M_{0,\epsilon}$}
\label{sec:m0def}

Let $M_{0,\epsilon}$ be an MDP with the same state and action spaces as $M_{a^\star,\epsilon}$. 
The transitions will be as with $M_{a^\star,\epsilon}$ too, except there is no special action $a^\star$ that leads to a game-over state.
Formally, letting $g_0:\SA \to \S$ denote the function that gives the next state for any given state-action pair, we have
\begin{align*}
g_0(s,a) = 
    \begin{cases}
    f_{H+1}\,, & \text{if } |s| \ge H\,;\\
    f_{|s|+1}\,, & \text{if } |s|<H \text{ and } \left(s\in \cF\setminus \{f_H\} \text{ or } a\in s \right)\,;\\
    s a\,, & \text{otherwise}\,.
    \end{cases}
\end{align*}
Let all reward distributions for all states and actions be $\cR_a(\cdot|s) = \delta_0$ (ie. rewards are always zero).
Note that this differs from the rewards of $M_{a^\star,\epsilon}$ only when action $a^\star$ is played or when a Bernoulli reward is given by $M_{a^\star,\epsilon}$ at stage $H$.

\subsection{Showing Eq.~(\ref{eq:problb})}
\label{sec:similarity}

Fix an arbitrary planner $\cP$ and $a\in \cA$
and let the distributions induced by the interconnection of $\cP$ 
and $M_{a,\epsilon}$ ($\cP$ and $M_{0,\epsilon}$) be denoted by $\bbP_a$ ($\bbP_0$, respectively).
We claim that for any $n\ge 1$,
\begin{align}
\bbP_a(a\not \in A_{1:n}) \ge (1-\epsilon)^n \bbP_0(a\not\in A_{1:n})
\label{eq:epsproblb}
\end{align}
holds. 
Fix $n\ge 1$.
Notice that the 
support of the reward distributions 
in either $M_{a,\epsilon}$ or $M_{0,\epsilon}$ is finite. Let $\cW$ be this set.
By the law of total probability,
\begin{align*}
\MoveEqLeft
\bbP_a(a\not\in A_{1:n}) \\
&= 
\sum_{a_{1:n}\in (\cA\setminus \{a\})^n}
\sum_{\substack{r_{1:n}\in \cW^n\\
s_{1:n},s'_{1:n}\in \cS^n\\
c_{1:n}\in \{0,1\}^n}}
\bbP_a( 
A_{1:n}=a_{1:n},
R_{1:n}=r_{1:n},
S_{1:n}=s_{1:n},
S_{1:n}'=s_{1:n}',
C_{1:n}=c_{1:n}
)
\end{align*}
and the same holds for $\bbP_0$.
Hence, it suffices to show that for any fixed 
$a_{1:n}\in (\cA\setminus \{a\})^n$,
$r_{1:n}\in \cW^n$, $s_{1:n},s'_{1:n}\in \cS^n$ 
and $c_{1:n}\in \{0,1\}^n$
such that 
$\bbP_0( 
A_{1:n}=a_{1:n},
R_{1:n}=r_{1:n},
S_{1:n}=s_{1:n},
S_{1:n}'=s_{1:n}',
C_{1:n}=c_{1:n}
)\ne 0
$,
\begin{align*}
\kappa:=
\frac
{\bbP_a(
A_{1:n}=a_{1:n},
R_{1:n}=r_{1:n},
S_{1:n}=s_{1:n},
S_{1:n}'=s_{1:n}',
C_{1:n}=c_{1:n}
)}
{\bbP_0(
A_{1:n}=a_{1:n},
R_{1:n}=r_{1:n},
S_{1:n}=s_{1:n},
S_{1:n}'=s_{1:n}',
C_{1:n}=c_{1:n}
)}
\ge (1-\epsilon)^{n}
\end{align*}
holds.
By construction (see Section~\ref{sec:inducedpd}), 
both the numerator and the denominator factorizes into the product of $n$ terms.
In fact, in both the numerator and the denominator, the terms coming from the planner $\cP$ are identical and hence cancel.
Let $g$ be the transition function in $M_{a,\epsilon}$.
By our assumption that the probability in the denominator is nonzero and because
 $a\not\in a_{1:n}$ we have $s_t'=g(s_t,a_t) = g_0(s_t,a_t)$. Hence,
\begin{align*}
\kappa = \prod_{t=1}^n \frac
{\bbP_a(R_t=r_t \mid S_t=s_t,A_t=a_t)}
{\bbP_0(R_t=r_t \mid S_t=s_t,A_t=a_t)}
\,.
\end{align*}
For $a'\ne a$, the reward distributions in $M_a$ and $M_0$ only differ when $s\in \cA^{H-1}$.
Thus, the ratio above is one unless $s_t\in \cA^{H-1}$.
Now, since $\bbP_0(R_t=r_t|S_t=s_t,A_t=a_t)> 0$, 
 $r_t = 0$. Therefore,
 $\bbP_a(R_t=r_t|S_t=s_t,A_t=a_t) \ge 1-\epsilon$ by Eq.~(\ref{ineq:bernoulli-small}).
 Hence, $\kappa \ge (1-\epsilon)^{n}$, as required, finishing the proof of Eq.~(\ref{eq:epsproblb}).

It remains to be shown that Eq.~(\ref{eq:problb}) holds for a suitable choice of $n$.
In fact, we we choose
\begin{align}
n = \left\lfloor \frac{k}{4} \wedge \left(\frac1\epsilon-1\right)/3.5 \right\rfloor\,.
\label{eq:nchoice}
\end{align}
Note that by Eq.~(\ref{eq:eps-less-than-one}), the second term in the minimum above is nonnegative.
On the one hand, this choice of $n$ implies that $n\le \lfloor k/4 \rfloor$, as required by 
Lemma~\ref{lem:mainlemma}.
Furthermore, $n\le \left(\frac1\epsilon-1\right)/3.5$ and hence, $1-\epsilon \ge 1- \frac{1}{1+3.5 n}$.
Hence,
\begin{align*}
(1-\epsilon)^n \ge \left(1- \frac{1}{1+3.5 n}\right)^n \ge \lim_{n\to\infty}  \left(1- \frac{1}{1+3.5 n}\right)^n > 3/4\,.
\end{align*}

\subsection{Proving Theorem~\ref{thm:lb}}
\label{sec:mainproof}
Let us now collect the choice of the parameters.
We have for $d\ge 18$, $0<\eta \le \frac{1}{2}-\frac{2}{\log_2(d-1)}$ (cf. Eq.~(\ref{eq:etarange})):
\begin{align}
\gamma & =(d-1)^{-\frac{1}{2}+\eta}\, \label{eq:gamma-choice}
\\
k            & =\floor{e^{\tfrac{(d-1)^{2\eta}}{8}}}\,, \label{eq:k-choice}
 \text{ and}\\
\epsilon & =\frac{1}{3}\dhpow{-H} \label{eq:eps-choice}
\end{align}
(cf. Eqs.~(\ref{eq:kdef}), (\ref{eq:gammdef}), and (\ref{eq:epdef})).
We have shown that the conditions of Lemma~\ref{lem:mainlemma} are satisfied when $n$ is set as shown in
Eq.~(\ref{eq:nchoice}). Thus, the theorem follows from the conclusion
of this lemma and 
Proposition~\ref{prop:delta-convert-tilde-delta},
once we lower bound $n$.
For this we have
\begin{align*}
n= \floor{k/4 \wedge (\epsilon^{-1}-1)/3.5}
&\ge\left(\exp\left((d-1)^{2\eta}/8\right)/4 \wedge \left(3\dhpow{H}-1\right)/3.5\right)   - 2\\
&=\left(\exp\left((d-1)^{2\eta}/8\right)/4 \wedge \frac{3}{3.5}\left( \frac{1}{2}(d-1)^{\frac{1}{2}-\eta}-\frac{1}{2} \right)^H-\frac{1}{3.5}\right)   - 2\\
&=\Omega\left(\min\left\{ \exp\left((d-1)^{2\eta}/8\right) , 2^{-H} d^{H\left(\frac{1}{2}-\eta\right)} \right\}\right)
\,,
\end{align*}
finishing the proof of the Theorem~\ref{thm:lb}.
\qed

\subsection{Lower bound for the discounted MDP case}\label{sec:discounting}

The construction and proof presented can be applied with minor modifications to the discounted MDP setting, which we briefly discuss here. In this setting, instead of maximizing the expected sum of rewards over a finite horizon $H$ (ie. $\E_\pi\left[\sum_{t=h}^H r_{A_t}(S_t) \mid S_h = s \right]$), the goal of the agent is to maximize the expected total discounted reward for some ``discount factor'' $0<\alpha<1$: $\E_\pi\left[\sum_{t=h}^\infty \alpha^{t-h}r_{A_t}(S_t) \mid S_h = s \right]$. 
Usually, this discount factor is close to $1$, and to make our results translate into this setting, we further assume that 
\[
\frac23\le \alpha < 1 \,.
\]
The Bellman equation Eq. (\ref{eq:bellman1}) is redefined to reflect this as:
\begin{align}
q^\star_h(s,a) &= r_a(s) + \alpha\int v^\star_{h+1}(s') P_a(ds'|s)  \,.\label{eq:bellman1-discounted}
\end{align} 
Fix some $H\ge1$ integer. For some action $a^\star$, let the MDP $M'_{a^\star, \epsilon}$ have the same states, actions, and transition structure (Eq. (\ref{eq:state-transition-def})) as $M_{a^\star, \epsilon}$.
Choose parameters $\gamma, k, \epsilon$ as before (Eqs.~(\ref{eq:gamma-choice}), (\ref{eq:k-choice}), and (\ref{eq:eps-choice})).
We define the feature map for this MDP $\phi'$ for $h\in[H],\,s\in\cS,\,a\in\cA$ as:
\[
\phi'_h(s,a)=\alpha^{-h+1}\phi_h(s,a)\,,
\]
where $\phi_h(s,a)$ is defined as before (Eq. (\ref{eq:phi-choice})), noting that for $h>H$, the agent is always in the absorbing stage $f_{H+1}$, so we simply define $\phi'_h(s,a)=\zeros$ for any $h>H$.
Let us redefine the final-stage rewards of Eq. (\ref{eq:mu-def}) to accommodate the discounting:
\begin{align*}
\label{eq:mu-def}
\mu_a(s) & = \alpha^{-H+1}\left(\epsilon\sigma_{s,a} \, \dhpows{} \ipg{ v_a ,v_{a^\star}} + \epsilon/2\right)\,.
\end{align*}
Note that as before, we have $\mu_a(s)=\ip{\phi'_{|s|}(s,a),\theta^\star}$ for $s\in\cA^{H-1}$ and $a\ne a^\star$.
Let us otherwise keep the reward distribution the same, according to Eq. (\ref{eq:reward-def}), using the new definition of $\mu_a(s)$.
It is easy to see that realizability in $M'_{a^\star, \epsilon}$ still holds, as we negated the effect of discounting in the Bellman equation with our new definitions.

Let us turn to showing that the rewards are still in $[0,1]$ and that $\mu_s(a)\in[0,1]$.
Notice that by the same argument as before, all rewards and $\mu_s(a)$ are non-negative, so we only need to show boundedness from above by 1.
Furthermore, as there can only be up to one non-zero reward in any episode, it suffices to show an upper bound on $q^\star_h(s,a)$ for any $s\in\cS,\,a\in\cA$, and $h=|s|$:
\[
\ip{\phi'_{h}(s,a),\theta^\star}=\epsilon \alpha^{-h+1}\left(c_h + \sigma_{s,a}\dhpow{H-h+1}\ip{v_a,v_{a^\star}}\right)\le 1\,.
\]
In two parts, first we show that
\begin{align*}
\epsilon \alpha^{h-1} c_h 
&= \epsilon \alpha^{-h+1}\left(\frac12 + \frac{1+\gamma}{2}\sum_{l=1}^{H-h}\dhpow{l} \right) 
< \epsilon \alpha^{-h+1} \frac{1+\gamma}{2}\left(\sum_{l=0}^{H-h}\dhpow{l} \right) \\
&\le \epsilon \alpha^{-h+1} \frac{1+\gamma}{2} 2 \dhpow{H-h+1} 
= \epsilon \left(\alpha\dhpow{}\right)^{-h+1} (1+\gamma) \dhpow{H} \\
&= \frac13 \left(\alpha\dhpow{}\right)^{-h+1}(1+\gamma) 
\le \frac{5}{12} \left(1.5\alpha\right)^{-h+1} < \frac12\,,
\end{align*}
where we repeatedly used that $\dhpow{}\ge 1.5$ (due to $\gamma\le\frac{1}{4}$) and $\alpha\ge \frac23$, and substituted the value of $\epsilon$.
Second, we show for the second term of the sum that
\begin{align*}
\epsilon \alpha^{-h+1} \sigma_{s,a}\dhpow{H-h+1}\ip{v_a,v_{a^\star}} 
&\le \epsilon \alpha^{-h+1} \dhpow{H-h+1}\\
&= \epsilon \left(\alpha\dhpow{}\right)^{-h+1} \dhpow{H}
\le \frac13 \left(1.5\alpha\right)^{-h+1}
< \frac12\,,
\end{align*}
where we used the bound on $\sigma_{s,a}$ (Eq. (\ref{ineq:sigma-bounds})).

The discounting also affects Eqs. (\ref{ineq:bernoulli-small}) and (\ref{eq:epsproblb}), which respectively become:
\begin{align*}
0\le \mu_a(s) \le \alpha^{-H+1}\epsilon\,
\end{align*}
~\vspace{-8mm}
\begin{align*}
\bbP_a(a\not \in A_{1:n}) \ge (1-\alpha^{-H+1}\epsilon)^n \bbP_0(a\not\in A_{1:n})\,.
\end{align*}
Carrying this change forward in the calculations, by a choice of $\eta=\frac{1}{2}-\frac{2}{\log_2(d-1)}$, analogously to Corollary~\ref{cor:eta-choice}, we arrive at
\[
\cC_\alpha(\cP,d)=\min(e^{\Omega(d)},\Omega((2\alpha)^H))
\,,
\]
where $\cC_\alpha(\cP,d)$ is defined for the $\alpha$-discounted setting analogously to Definition~\ref{def:query-cost} (with the difference of taking the supremum over any $q^\star$-realizable $(M,\phi)$ with discount factor $\alpha$ instead of horizon $H$). This holds for any $H\ge 1$, so due to our assumption of $\alpha\ge \frac23$, we can take the limit as $H\to\infty$, to get that
\[
\cC_\alpha(\cP,d)=e^{\Omega(d)}
\,.
\]

\section{Upper bound}\label{sec:upper}

The single-stage setting (i.e., when $H=1$) under $q^\star$-realizability 
is effectively a ``linear bandit'' problem 
\citep{LaSze19:book}. 
In this case $\delta$-good actions 
(i.e., actions $a$ with $q^\star(s,a)\ge v^\star(s)-\delta$)
can be found regardless of the cardinality of $\SA$ using $\mathcal{O}(\poly(d)/\delta^2)$ queries and computation time by simply choosing $\tilde{\mathcal{O}}(d)$ state-action pairs that provide maximal information about any other state-action pairs (via a so-called \textit{approximate $G$-optimal design}, cf. Proposition~\ref{prop:optimal-design}), followed by estimating the unknown linear parameter using a least-squares estimator. %

In this section, we consider the  ``na\"ive'' planning algorithm which treats the finite-horizon planning problem in MDPs as a sequence of $H$ single-stage problems, and applies the above procedure recursively. More specifically, the problem can be seen as as a sequence of misspecified linear bandits, with the misspecification of any given level being the estimation error of the previous level. We provide an upper bound on the query complexity of the planner which implements this algorithm -- we find that, perhaps unsurprisingly, the estimation errors compound multiplicatively over the different stages: by level $h \in [H]$, the learned hypothesis will roughly have an error of $\mathcal{O}\left((\sqrt{d})^{H-h+1}\right)$. The complementary lower bound of Section~\ref{sec:lower} tells us that in a certain sense this is the best that a planner can do (i.e., some multiplicative compounding of errors is unavoidable).

In more detail, our algorithm first estimates the optimal value functions $q^\star_H(s,a) = r_a(s)$ using a least-squares estimator and design points from $(s,a)\in \cS_H \times \cA$.
Recursively, once the estimate $f_{h+1}$ is computed, level $h$ can be treated as a single-stage problem with immediate rewards $\mu^{(h)}$ so that $\mu^{(h)}_a(s) = r_a(s) + \int \max_{a'}f_{h+1}(s',a') \cP_a(ds'|s)$. 
This gives a misspecified linear bandit, where the misspecification error is the estimation error of the preceding level,  $\norm{f_{h+1}-q^\star_{h+1}}_\infty$. 
 Readers familiar with the literature recognize that the algorithm described is known as the least-squares value iteration algorithm. 
The pseudo-code makes reference to Proposition~\ref{prop:optimal-design}, which we present after the algorithm. Below, we write $\Pi_H(x) \coloneqq \max(\min(x,H),-H)$ for the clipping operator. 
\begin{algorithm}
\begin{itemize}
\item Let $f_{H+1} \equiv 0$
 \item for $h=H$ downto $1$:
    \begin{enumerate}[(i)]
    \item Compute the experimental design $\rho_h \in \Dists( \S_h \times \A )$ satisfying the assumptions of Proposition~\ref{prop:optimal-design} 
    \item For each $z=(s,a) \in \supp(\rho_h)$, collect $n$ samples of $(R_i(z),S'_i(z)) \sim Q_a(\cdot,\cdot|s)$ and calculate the empirical average $ \hat \mu_h(s,a) = \frac{1}{n} \sum_{i=1}^n R_i(z) + \max_{a'} f_{h+1}(S'_i(z),a')$.
    \item Calculate the least-squares estimator $\hat \theta_h$ appearing in Proposition~\ref{prop:optimal-design} using $\rho=\rho_h$ and $r=\hat \mu_h$, and set $f_h(s,a) = \Pi_H\left(\langle \phi_h(s,a), \hat \theta_h \rangle\right)$.
    \end{enumerate}
\item Return $(f_h)_h$. %
\end{itemize}
\caption{Least-squares value iteration with $G$-optimal design}
\label{alg:lsvi}
\end{algorithm}

We say that $\mu:\cX \to \bR$ is an $\epsilon$\textit{-realizable linear function} for feature map $\phi' : \cX \rightarrow \bR^d$ 
if there exists a vector $\theta^\star \in \bR^d$ such that 
\[
\mu(x) = \langle \phi(x), \theta^\star \rangle + \epsilon_x, \quad |\epsilon_x| \leq \epsilon \ \forall x \in \cX .
\]
To perform the least-squares estimation, we follow the approach of \cite{LaSzeGe19} and make use of the \textit{Kiefer-Wolfowitz} theorem to sample only from a few state-action pairs per horizon. The main result we use is the following, which bounds the error of producing a least-squares estimator over a certain distribution of points with small support: 

\begin{proposition}\label{prop:optimal-design}
Assume that $\{\phi'(x)\,:\, x\in \cX \}\subset \R^d$ is compact.
There exists a distribution $\rho \in \Dists(\cX)$ whose support has at most $4d \log\log d + 16 $ points such that for any $r,\mu:\cX \to \R$,
the vector
\begin{align*}
\hat \theta = G(\rho)^{-1} \sum_{x\in \supp(\rho)} \rho(x) r(x) \varphi'(x), \quad \text{ where } G(\rho) = \sum_{x\in \supp(\rho)} \rho(x) \varphi'(x)\varphi'(x)^\top
\end{align*}
satisfies
\[
\sup_{x\in \cX} \left|\langle \phi'(x),\hat \theta\rangle-\mu(x) \right| \leq \epsilon + (\epsilon +\delta) \sqrt{2d}
\]
provided that
$\mu:\cX \to \R$ is an $\epsilon$-realizable linear function with feature map $\phi'$
and $ \|r-\mu\|_\infty\le \delta$. 
\end{proposition}
The distribution $\rho$ appearing in the above proposition is called a \textit{near-optimal experimental design}. Note that the result only needs $r$ to be specified at the points in the support of $\rho$, a property that we use in our algorithm.
\begin{proof}
This is essentially Proposition~4.5 of \cite{LaSzeGe19} with the difference that here we allow infinite $\cX$. 
As the Kiefer-Wolfowitz theorem applies regardless the cardinalities of the sets involved, the result continues to hold in our case.
\end{proof}
\noindent The following result gives an upper bound on the estimation error and query complexity of this algorithm. 

\begin{proposition}\label{prop:backwards-error}
Let $f_h:\cS_h\times \cA \to \R$, for $h \in [H]$, be the functions returned by Algorithm~\ref{alg:lsvi}.
Consider the planner $\cP$ that chooses actions greedily based on $f=(f_h)$. 
Let $(M,\phi)\in\mathcal{M}_{H,d}$ be $q^\star$-realizable and bounded (Definition~\ref{def:bounded-realizable-mdp}). Then setting $n=\tilde \Theta( H^4 (2d)^H /\delta^2 )$, 
$\cP$ is a $\delta$-sound planner that uses at most
\[
\cC(\cP,H,d) = \tilde{\mathcal{O}}\left(\frac{H^5 (2d)^{H+1}}{\delta^2}\right)
\]
queries.
\end{proposition}
In fact, the proof shows that the planner is not only a local $\delta$-sound planner, but it is also a \textbf{global} $\delta$-sound planner. That is, no matter the state, after the said number of queries, the planner can output the actions of a $\delta$-optimal policy without any further interactions with the simulator (and the cost of this computation is $\mathcal{O}(d|\cA|)$). We also note that the upper bound is exponential in $H$, as is the second term of the lower bound. Still there remains a gap between the lower and upper bounds.

For the proof, we will need a finite-horizon analogue of Corollary~2 from \cite{SinghYee94}. For any $f: \SA \rightarrow \bR$, let $\pi_{f}(s) = \argmax_{a} f(s,a)$ denote a policy that is greedy with respect to $f$.
\begin{lemma}\label{lemma:finite-singh}
Given any $(f_h)_h$ with $f_h:\cS_h \times \cA \to \R$, the $H$-stage policy $\pi$ that at stage $h\in [H]$ is greedy with respect to $f_h$ ($\pi_h = \pi_{f_h}$) satisfies
$\max_h \norm{v_h^\star-v_h^{\pi}}_\infty \leq 2 \sum_{h=1}^H\norm{q^\star_h-f_h}_\infty$.
\end{lemma}
\begin{proof}
Let $s \in \S_1$ and let $\pi^\star$ be an optimal policy.
From the Bellman equations for $\pi^\star$ and $\pi$,
\begin{align*}
    v^{\star}_1(s) - v^{\pi}_1(s) 
    &= q^\star_1(s,\pi^\star) - q^\star_1(s,\pi_{f_1}) + q^\star_1(s,\pi_{f_1}) - q^\pi_1(s,\pi_{f_1}) \\
    &= q^\star_1(s,\pi^\star) - q^\star_1(s,\pi_{f_1}) + \ip{P_{\pi_{f_1}}(s), v_2^\star-v_2^\pi} \\
    &\leq q^\star_1(s,\pi^\star) - f_1(s,\pi^\star)+ f_1(s,\pi_{f_1})- q_1^\star(s,\pi_{f_1}) + \ip{P_{\pi_{f_1}}(s), v_2^\star-v_2^\pi} \\
    &\leq 2\norm{q_1^\star-f_1}_\infty + \norm{v_2^\star - v_2^{\pi}}_\infty\,,
\end{align*}
where the first inequality follows from that $\pi_{f_1}$ is greedy with respect to $f_1$.
An inductive argument completes the proof. 
\end{proof}
We also need Hoeffding's inequality:
\begin{lemma}[Hoeffding's inequality] 
\label{lem:hoeffding}
Let $(X_i)_{i\in [m]}$ be i.i.d. random variables from the unit interval, $[0,1]$.
Letting $\mu$ denote their common mean, 
for any $\zeta\in (0,1]$ it holds with probability $1-\zeta$ that
\begin{align*}
\absg{\frac1m \sum_{i=1}^m X_i - \mu }\le \sqrt{\frac{\log(2/\zeta)}{2m}}\,.
\end{align*}
\end{lemma}
\begin{proof}(of Proposition~\ref{prop:backwards-error}).
 We first derive an error bound on how well the optimal action-value function is approximated by the functions computed. 
 Let $m=\lceil 4d \log\log d + 16 \rceil$ be the maximum cardinality of the support of $\rho_h$.
Recall that $f_{H+1}=0$.
For $h\in [H+1]$, let $g_{h}(s) = \max_a f_h(s,a)$ and
\begin{align*}
 \mu_h(s,a) &= r_a(s) + \ip{P_a(s),g_{h+1}}\,.
\end{align*}
Fix $h\in [H]$, $z = (s,a)\in \supp(\rho_h)$,
\begin{align*}
\hat \mu_h(s,a) - \mu_h(s,a) = 
\frac1n \sum_{i=1}^n \underbrace{R_i(z)-r_a(s) + g_{h+1}(S'_i(z))-\ip{P_a(s),g_{h+1}}}_{X_i}\,.
\end{align*}
 Note that $X_i \in [-H, H]$ as rewards are positive (Assumption~\ref{ass:bounded-mdp}), and the output of $f_h$ is clipped at $[-H,H]$ (cf. Step (iii) of Algorithm~\ref{alg:lsvi}). Also note that $X_i$ has mean zero. Indeed,
 \begin{align*}
\EE{X_i} =\EEg{ \EE{ X_i | g_{h+1} } }= \EE{R_i(z)-r_a(s)} + 
\EEg{ \EE{ g_{h+1}(S'_i(z))-\ip{P_a(s),g_{h+1}} | g_{h+1} }} = 0\,.
\end{align*}
Hence, by a union bound and Hoeffding's inequality,
 with probability at least $1-\zeta$,
 we have 
 \begin{equation}\label{eq:beta}
 |\hat \mu_h(s,a) - \mu_h (s,a) | \leq H \sqrt{\frac{2}{n}\log\left(\frac{2H m}{\zeta}\right)}\coloneqq \beta \quad \forall h \in [H], \forall (s,a) \in \supp(\rho_h).
 \end{equation}
In the rest of the proof we work on the event where the inequalities in the above display hold (i.e., the inequalities below hold with probability at least $1-\zeta$).
 By backwards induction, we will show that the error at level $h$ satisfies
 \begin{equation}\label{eq:backwards-error}
 \norm{f_{h}-q^\star_{h} }_\infty \leq \beta \left((2+\sqrt{2d})^{H-h+1} - 1 \right) \coloneqq \epsilon_{h} 
 \end{equation}
 Starting with stage %
 $H$, by realizability of $(M,\phi)$, we can apply Proposition~\ref{prop:optimal-design} with $\delta=\beta$, $\epsilon=0$, $\mu=r=q^\star_H$, and $\phi'=\phi_H$ to get
\[
 \norm{f_H - q^\star_H}_\infty \leq \beta \sqrt{2d},
\]
 as desired.
Now let $1\le h < H$ and assume the claim is true for level $h+1$; 
we  will now show that it is also true for level $h$. 
We claim that $\mu_h$ is $\epsilon_{h+1}$-realizable under $\phi_h$.
Indeed, from the definition of $\mu_h$, the realizability of $q_h^\star$ and the Bellman equation,
\begin{align*}
\abs{\mu_h(s,a)-q_h^\star(s,a)} &= 
\absg{\int P_a(ds'|s) \left\{\max_{a'} f_{h+1}(s',a') - \max_{a'} q_{h+1}^\star(s',a') \right\}} \\
&\le \norm{f_{h+1} - q_{h+1}^\star}_\infty \le  \epsilon_{h+1}\,,
\end{align*}
where the last inequality follows by the induction hypothesis.
Applying Proposition~\ref{prop:optimal-design} again, with $\epsilon=\epsilon_{h+1}$,
$\delta=\beta$,
$\mu=\mu_h$, 
and $\phi'=\phi_h$
 gives:
\begin{align*}
\norm{f_h - q^\star_h}_\infty 
& \le \norm{f_h - \mu_h}_\infty +\epsilon_{h+1} \\
& \le \epsilon_{h+1}(2+\sqrt{2d}) + \beta\sqrt{2d} \\
&= \beta \left((2+\sqrt{2d})^{H-h} - 1 \right)(2 + \sqrt{2d}) + \beta \sqrt{2d} \\
&= \beta \left((2+\sqrt{2d})^{H-h+1} - 1\right), 
\end{align*}
as desired. 
Letting $\pi$ be the policy implemented by the planner, we bound its suboptimality using  Lemma~\ref{lemma:finite-singh}: 
$
\norm{v^\star_1-v^{\pi}_1}_\infty \leq 2H\norm{q^\star_1-f_1}_\infty = 2H\beta\left((2+\sqrt{2d})^{H}-1)\right).%
$
By boundedness of rewards (Assumption~\ref{ass:bounded-mdp}), on the error event that has probability $\zeta$ at most, the policy computed incurs a 
total cost of at most $H$ compared to the optimal policy, %
it holds that $\pi$ is at least $2H\beta\left((2+\sqrt{2d})^{H}-1\right) + H\zeta$-optimal.
Setting $\zeta = \delta/2$ and setting $n$ so that $2H\beta\left((2+\sqrt{2d})^{H}-1\right)\le \delta/2$ also holds, we get that with 
\begin{align*}
n = \left\lceil
32 H^4 \log\left(\frac{4H m}{\delta}\right) \left((2+\sqrt{2d})^{H}-1\right)^2 \delta^{-2} \right\rceil
\end{align*}
queries per the $mH$ support points, the planner $\cP$ is $\delta$-sound.
Plugging in the value of $m$ and simplifying gives the desired result.
\end{proof}
As a side note we remark that the proof avoids a union bound over the value functions. This works because the random next states that are used in the random value function computed are chosen independently of the random value function.

Note that not accounting for the computation of the distributions $(\rho_h)_h$, the total computation cost of a naive implementation of the planner for a single stage $h\in [H]$ 
is $\mathcal{O}(n m |\cA| d + m d + d^3) = \tilde{\mathcal{O}}(H^4 d^2 |\cA| + d^3)$, making the total compute cost
$\tilde{\mathcal{O}}(H^5 d^2 |\cA| + H d^3)$. When a quadratic optimization oracle is available over the set 
$\{\phi_h(s,a)\,:\, (s,a)\in \cS_h \times \cA\}$, a version of the Frank-Wolfe algorithm can be used to approximately compute $\rho_h$ in $\text{poly}(d)$-time
\citep[Notes 3-4, Chapter 21.2]{LaSze19:book}.

\section{Discussion}  
\label{sec:disc}
In this paper we have shown an exponential lower bound on the query complexity 
of local planning with linear function approximation in both fixed-horizon and discounted MDPs, under the assumption that the optimal action-value function is realizable by the features. We have also given an upper bound for the fixed-horizon setting, which, in some regimes of the parameters 
is relatively close to the lower bound.
Closing the gap between the lower and upper bounds remains an interesting open problem.
Since the upper bound applies to global planning and not only local planning, there is a possibility that the query complexity of local and global planning in this specific setting are the same. It would also be interesting to refute or validate this conjecture.

Our lower bound construction critically relied on allowing the action set to be exponentially large in the dimension.  %
We remark that the large action set does not preclude a polynomial sample complexity -- in particular, \citet{LaSzeGe19} give a polynomial upper bound regardless of the number of actions, under the assumption that the state-action value function $q^\pi$ for any memoryless policy $\pi$ is realizable. 
Note that in the single-stage/linear bandit setting (with $H=1$), this assumption becomes equivalent to the realizability of $q^\star$. However, for $H>1$, our exponential sample complexity lower bound shows a strong separation between these two realizability assumptions.
The difficulty of planning (and learning) thus cannot be attributed to the large action set alone.
Yet, it is intriguing that our construction relies heavily on the action set being large: it remains an interesting open question to resolve the worst-case query complexity of $q^\star$-realizable planning with linear function approximation when the action set is of constant size. %

\section{Acknowledgements}

We thank Nan Jiang for insightful discussions, and Roshan Shariff for helpful feedback on an earlier draft. PA gratefully acknowledges funding from the Natural Sciences and Engineering Research Council (NSERC). 
CS gratefully acknowledges funding  from 
the Canada CIFAR AI Chairs Program, Amii and NSERC.

\bibliography{linear_fa}

\end{document}